\documentclass[11pt]{elsarticle}
\usepackage[english]{babel}
\usepackage[letterpaper,top=2cm,bottom=2cm,left=3cm,right=3cm,marginparwidth=1.75cm]{geometry}

\usepackage{macros}

\usepackage{amsmath}
\usepackage{amssymb}
\usepackage{amsthm}
\usepackage{graphicx}
\usepackage{verbatim}

\newcommand{\D}{\nabla}
\newcommand{\dx}{\Delta \theta_{nt}}
\newcommand{\C}{\choose}
\newcommand{\se}{\sqrt{\epsilon}}

\newcommand{\thetahat}{\widehat{\theta}}

\newcommand{\fbar}{\bar{f}}

\newcommand{\thetatil}{\widetilde{\theta}}
\newcommand{\dxtil}{\Delta \thetatil_{nt}}

\newtheorem{assumption}{Assumption}

\begin{document}

\begin{frontmatter}

\title{Robust empirical risk minimization via Newton's method}

\author[1]{Eirini Ioannou}
\ead{ei250@cantab.ac.uk}

\author[2]{Muni Sreenivas Pydi}
\ead{muni.pydi@lamsade.dauphine.fr}

\author[3]{Po-Ling Loh\corref{cor1}}
\ead{pll28@cam.ac.uk}

\cortext[cor1]{Corresponding author}
\affiliation[1]{organization={University of Edinburgh},
country = {United Kingdom}}
\affiliation[2]{organization={LAMSADE, University of  Paris Dauphine-PSL},
country = {France}}
\affiliation[3]{organization={University of Cambridge},
country = {United Kingdom}}

\begin{abstract}
A new variant of Newton's method for empirical risk minimization is studied, where at each iteration of the optimization algorithm, the gradient and Hessian of the objective function are replaced by robust estimators taken from existing literature on robust mean estimation for multivariate data. After proving a general theorem about the convergence of successive iterates to a small ball around the population-level minimizer, consequences of the theory in generalized linear models are studied when data are generated from Huber's epsilon-contamination model and/or heavy-tailed distributions. An algorithm for obtaining robust Newton directions based on the conjugate gradient method is also proposed, which may be more appropriate for high-dimensional settings, and conjectures about the convergence of the resulting algorithm are offered. Compared to robust gradient descent, the proposed algorithm enjoys the faster rates of convergence for successive iterates often achieved by second-order algorithms for convex problems, i.e., quadratic convergence in a neighborhood of the optimum, with a stepsize that may be chosen adaptively via backtracking linesearch.
%We provide a new computationally-efficient algorithm that finds estimators for the risk minimization problem. We show that these estimators are robust for general statistical models, under the robustness setting for the classical Huber $\epsilon$- contamination model. Our workhorse is a novel robust variant of, Newton’s method and we provide conditions under which our newton’s method variant provides accurate estimators in a general convex risk minimization problem. We provide specific consequences of our theory for linear regression. Finally, we study the empirical performance of our proposed methods on synthetic datasets, and find that our methods convincingly outperform a variety of baselines.
\end{abstract}

\end{frontmatter}

%%%%%

\section{Introduction}

Statistical estimation via classical procedures often depends on strong model assumptions, which only hold in the absence of outliers and other deviations. However, many real-life data sets do not typically follow these model assumptions, necessitating the use of robust statistical methods~\cite{huber2011, rousseeuw2011robust, maronna2019robust}, which remain reasonably accurate even under deviations from the model assumptions. In this paper, we focus on situations where data are sampled from a small ball around a parametric distribution, according to Huber's $\epsilon$-contamination model. In other words, we have samples of the form $z_i \sim (1-\epsilon)P_{\theta^*}+\epsilon Q$, where $Q$ is an arbitrary distribution and the goal is to estimate the unknown parameter $\theta^*$ based on an observed data set $\{z_i\}_{i=1}^n$. We also analyze the behavior of the same algorithms in situations where data are generated from a heavy-tailed distribution. Although the parameter corresponds to the true data-generating distribution, ``outliers" are observed in the data set due to random sampling, and the goal is to obtain an estimator with similar high-probability guarantees as in the case of standard parameter estimation techniques for lighter-tailed distributions.

Classical robust statistics~\cite{huber2011} suggests the use of $M$-estimators, which involve optimizing an appropriate loss function over the space of parameters. More specifically, suppose we wish to estimate the parameter $\theta^* = \arg\min_{\theta \in \Theta} \mathcal{R}(\theta)$, where the risk $\mathcal{R}(\theta) = \E[\mathcal{L}(\theta,(x,y))]$ is the expectation of a loss function. In practice, one uses an empirical risk minimizer $\thetahat \in \arg\min_{\theta \in \Theta} \frac{1}{n} \sum_{i=1}^n \mathcal{L}(\theta, (x_i, y_i))$. Standard theory of parametric statistics shows that the optimal choice of $\mathcal{L}$ corresponds to the log-likelihood function when data are not contaminated. However, taking into account $\epsilon$-contamination leads to the use of other losses such as the Huber loss, which can be shown to be optimal in a minimax sense when the uncontaminated data are normally distributed~\cite{huber2011}. Similarly, while the least-squares loss corresponds to maximum likelihood for Gaussian errors, minimizing a different loss function may be advantageous in the case of heavy-tailed data.
%A concrete example that we have worked on is linear regression data, i.e. data $(x,y)$ that follows the relation: $y=x^T \theta^* + w, w\sim \mathcal{N}(0,\sigma^2)$ with loss function $\mathcal{L}(\theta,(x,y))=\frac{1}{2}(y-x^T \theta)^2$. From this discussion we conclude that we need to solve a minimization problem, hence we could use optimization techniques to tackle it.

In this paper, we adopt an alternative approach inspired by optimization methods~\cite{boyd2004}. Rather than seeking to design a robust loss, we introduce robustness into the estimation algorithm by implementing robust updates in an iterative second-order optimization procedure. Our work is directly inspired by the work of Prasad et al.~\cite{prasad2018}, who proposed and analyzed a first-order version of this method. Our algorithm, which we call ``robust Newton's method," utilizes the $AgnosticMean$ algorithm from Lai et al.~\cite{lai2016} in the Huber contamination setting to obtain robust gradient and Hessian estimates on each iterate of our algorithm. Moreover, given appropriate assumptions, we prove that the rate of convergence of this algorithm is faster than that of robust gradient descent, and successive iterates converge quadratically to a small ball around $\theta^*$. Furthermore, a suitable stepsize may be determined adaptively using a robust variant of backtracking linesearch. Our analysis of the Newton iterates is fairly general, and can be used to derive convergence guarantees when alternative procedures are employed for gradient/Hessian estimation. We consequently propose a method based on the conjugate gradient method~\cite{wright1999numerical} for obtaining approximate Newton directions which may be useful in higher dimensions, and discuss some conjectures about the corresponding convergence rate on $\epsilon$-contaminated data.
%We conclude that, if we can estimate the Hessian of the risk function accurately then this method is more beneficial than the methods used so far to estimate $\theta^*$.
%Experimental results on synthetic data seem to agree with our theoretical results of convergence and accuracy.

%Even though this is a useful algorithm, there needs to be further investigation. For instance, there needs to be developed an efficient algorithm that can estimate the Hessian of the risk function. Some ideas towards that is using the conjugate gradient method (see~\cite{boyd2004}) as well as the formula: $\D^2 \mathcal{R} v = \lim_{\delta \rightarrow 0} \frac{\mathcal{R}(\theta+\delta v)-\mathcal{R}(\theta)}{\delta}$~\cite{james2010}.

\subsection{Related Work}

Here, we discuss several other general approaches to robust empirical risk minimization (ERM) which have appeared in the literature. A variety of algorithms have been proposed based on median-of-means estimators, which give robust alternatives to mean estimators (a more detailed description is provided in Section~\ref{SecHT} below). Median-of-mean tournaments~\cite{lugosi2019risk, lugosi2019regularization, lugosi2019mean} provide a method for comparing pairs of candidate regression functions based on the number of blocks in which the empirical mean of the loss function is smaller for one function than the other. The final estimator is a function which ``wins" the most pairwise matches among other candidate functions. Another use of median-of-means estimators derives an estimator by considering a ``minimaximization" problem formed by increments of the objective function, where a median-of-means estimate is used in place of the expectation appearing in the population-level version of the problem~\cite{lecue2019learning, lecue2020robusta, chinot2020robust}. Finally, and more similar in spirit to the approach taken in our paper, we mention a method which involves modifying gradient descent by computing a gradient with respect to a median block on each iteration~\cite{lecue2020robustb}. The median block is defined as the block with the smallest empirical mean (with respect to the objective function value) on the current iteration. Excess risk bounds are then derived for a class of binary classification problems, where a certain fraction of the data consists of arbitrarily generated outliers and the remaining points are drawn i.i.d.\ from the uncontaminated model.

The SEVER algorithm~\cite{diakoconf2019} also operates via an appropriate modification of an iterative optimization procedure. It uses any ``approximate learner" algorithm, which can find an approximate critical point of an empirical risk minimization problem, as a subroutine (e.g., gradient descent, stochastic gradient descent, or Newton's method). On successive iterations, the SEVER algorithm filters out data points by applying the approximate learner to the currently remaining set of data points and then filtering out any points with outlying gradients computed at the parameter chosen by the approximate learner. Statistical error bounds are derived for the output of the SEVER algorithm on classification and regression problems, where data are drawn from a possibly heavy-tailed model and then corrupted by a small fraction of adversarial outliers.

Finally, our work is most closely related to the work of Prasad et al.~\cite{prasad2018}, which may be seen as a first-order version of our second-order algorithm. In that paper, the authors propose to perform parameter estimation by running a variant of gradient descent on the empirical risk objective, where successive gradients are computed by treating each gradient computation as an approximation of a population-level mean, and then applying a robust mean estimation procedure for multivariate data. As in our work, they use the mean estimation algorithm by Lai et al.~\cite{lai2016} for their multivariate estimation procedure in the case of Huber's $\epsilon$-contamination model. They also derive statistical error bounds for successive iterates, which hold with high probability. The main difference with our work is that we are able to derive faster rates of convergence due to the use of second-order algorithms, while enjoying the broad applicability of their approach.

%%%%%

\subsection{Outline} 

The remainder of our paper is organized as follows: In Section~\ref{SecBackground}, we discuss the setup of the problem we are aiming to solve. In Section~\ref{SecMain}, we introduce our novel robust Newton's method and present two theorems concerning its convergence. In Section~\ref{SecApps}, we discuss applications of our general theory to generalized linear models. In Section~\ref{SecSims}, we provide some illustrative numerical results and comparisons. In Section~\ref{SecCG}, we present a version of robust Newton's method based on the conjugate gradient method and provide some conjectures. Finally, we conclude our paper with a discussion of open directions in Section~\ref{SecDiscussion}.

%%%%%

\subsection{Notation}

For a matrix $A \in \real^{p \times p}$, we use $\|A\|_2$ to denote the spectral norm, $\lambda_{\min}(A)$ to denote the minimum eigenvalue, and $\tr(A)$ to denote the trace. We use $c, C, c_1, C_1, c_2, C_2, \dots$ to denote universal positive constants whose specific values may change from line to line. For functions $f(n)$ and $g(n)$, we write $f(n) = O(g(n))$ to mean that $f(n) \le Cg(n)$ for some constant $C > 0$, and also write $f(n) \precsim g(n)$ and $g(n) \succsim f(n)$. We write $f(n) \asymp g(n)$ when both inequalities hold simultaneously. We use $\widetilde{O}(n)$ to hide logarithmic factors. We use the abbreviation ``w.h.p." for ``with high probability," meaning with probability tending to 1 as the sample size $n$ tends to $\infty$.
%For a data set $\{x_i\}_{i=1}^n \subseteq \real$, we write median$(\{x_i\}_{i=1}^n)$ to denote the median of the data points.

%%%%%

\section{Background}
\label{SecBackground}

%\subsection{Parametric Estimation}

We consider a parametric estimation problem, wherein the data $\{z_i\}_{i=1}^n \subseteq \cZ$ sampled from a true distribution $P$ are to be fit to a model with parameter $\theta\in \Theta$. A loss function $\cL: \Theta\times \cZ\to \real$ measures the goodness of fit of the model.
The optimal parameter $\theta^*\in \Theta$ minimizes the population risk of the model, which is the expected loss incurred by the model over the true data distribution:
\begin{align}\label{eq: ERM}
    \theta^* = \argmin_{\theta\in \Theta} \cR(\theta) \defn \E_{z\sim P}[\cL(\theta, z)].
\end{align}

Given $n$ i.i.d.\ data points $\{z_i\}_{i=1}^n$ sampled from the true distribution $P$, the goal in empirical risk minimization is to estimate the parameter $\widehat{\theta}_n$ that  minimizes the empirical risk of the model, which is the average loss incurred by the model over the $n$ data points:
\begin{align}
    \widehat{\theta}_n = \argmin_{\theta\in \Theta} \widehat{\cR}_n(\theta) \defn \frac{1}{n}\sum_{i=1}^n \cL(\theta, z_i).
\end{align}

\subsection{Examples}

\paragraph{\textbf{Linear regression:}} In linear regression, data $z\in \cZ$ are of the form $z = (x,y) \in \real^p \times \real$, where the covariate $x$ and response $y$ are related via
\begin{align*}
%\label{eq: model linear regression}
    y = x^T\theta^* + w,
\end{align*}
where $w\in \real$ is noise that is sampled independently from $x$ and $y$. The loss function we use for this model is the squared loss function, 
\begin{align*}
%\label{eq:loss linear regression}
    \cL(\theta, (x,y)) = \frac{1}{2}(y-x^T\theta)^2.
\end{align*}

\paragraph{\textbf{Generalized linear models:}} In a generalized linear model (GLM), data $z = (x,y) \in \real^p \times \real$ are sampled from a true distribution $P$ that satisfies the following relation on the conditional probability of $y$ given $x$:
\begin{align}\label{eq: model GLM}
    P(y|x) \propto \exp{\left(\frac{y x^T\theta^* - \Phi(x^T\theta^*)}{c(\sigma)} \right)},
\end{align}
where $c(\sigma)$ is the scale parameter and $\Phi: \real\to\real$ is a convex link function. The loss function we use for a GLM is the negative log-likelihood,
\begin{align}\label{eq:loss GLM}
    \cL(\theta, (x,y)) = - y x^T\theta + \Phi(x^T\theta).
\end{align}

\subsection{Optimization Algorithms}

In practice, we seek efficient algorithms for solving the ERM problem \eqref{eq: ERM}. A popular algorithm is gradient descent~\cite{bertsekas2015convex}. Given an initial guess for the parameter $\theta_0\in \Theta$ and a stepsize $\eta$, the gradient descent algorithm generates a sequence of iterates $\{\theta_t\}_{t=1}^\infty$, as follows:
\begin{align*}
    \theta_{t+1}=\theta_t - \eta \D \cR(\theta_t).
\end{align*}

Another popular algorithm is Newton's method~\cite{wright1999numerical, boyd2004}, whose iterates are given by the following update equation:
\begin{align}
\label{EqnNewtonUpdate}
    \theta_{t+1}=\theta_t-(\D^2 \mathcal{R}(\theta_t))^{-1}\D \mathcal{R}(\theta_t).
\end{align}

Whereas gradient descent uses only gradient information at the current iterate $\theta_t$, Newton's method uses both gradient and Hessian information at the current iterate. 

%%%%%

%\subsection{Robust Estimation}

\subsection{Huber's $\epsilon$-Contamination Model}
\label{SecHubEps}

In Huber's $\epsilon$-contamination model, samples are drawn from a mixture distribution of the form
\begin{align}\label{eq: huber model}
    P_\epsilon = (1-\epsilon)P + \epsilon Q,
\end{align}
where $P$ is the true data distribution and $Q$ is an arbitrary noise distribution. The goal is to estimate a parameter $\theta^*\in \Theta$ corresponding to the uncontaminated component $P$, given $n$ i.i.d.\ samples drawn from the corrupted distribution $P_\epsilon$.

Huber's contamination model is a classical model studied in robust statistics~\cite{huber2011, rousseeuw2011robust, maronna2019robust}, with many exciting theoretical breakthroughs in estimation and inference. More recently, as robust statistics received renewed attention in the theoretical computer science community, additional questions were raised, particularly concerning computational tractability for optimal robust estimators in high dimensions. The contemporaneous work of Lai et al.~\cite{lai2016} and Diakonikolas et al.~\cite{diakosiam2019} studied computationally tractable mean estimation in multivariate Gaussian settings, where the former paper studied contamination with respect to Huber's model and the latter paper studied a stronger form of ``adversarial" contamination. The subroutine which we call Algorithm~\ref{alg:Huber estimator} comes from Lai et al.~\cite{lai2016}---we state it in the slightly adapted version studied in Prasad et al.~\cite{prasad2018}.

\begin{algorithm}[h!]
\centering
\caption{Huber Estimator}\label{alg:Huber estimator}
\begin{algorithmic}[1]
\Require Samples $S = \{s_i\}_{i = 1}^n$, Corruption level $\epsilon$, Dimension $p$, Failure probability $\delta$
\Function{HuberEstimator($S = \{s_i\}_{i = 1}^n$, $\epsilon$,  $p$, $\delta$)}{}
\State Set $\widetilde{S} = \text{{\sc HuberOutlierTruncation}}(S,\epsilon,p,\delta)$
\If{$p=1$} 
\State \Return $\text{mean}(\widetilde{S})$
\Else
\State Compute $\Sigma_{\widetilde{S}}$, the covariance matrix of $\widetilde{S}$

\State Compute $V$, the span of the top $p/2$ principal components of $\Sigma_{\widetilde{S}}$, and $W$, its complement

\State Set $S_1 := P_V (\widetilde{S})$, where $P_V$ is the projection operation onto $V$

\State Set  $\widehat{\mu}_V := \text{\sc{HuberEstimator}}(S_1,\epsilon,p/2,\delta)$

\State Set $\widehat{\mu}_W := \text{mean}(P_W \widetilde{S})$

\State Set $\widehat{\mu} \in \mathbb{R}^p$ such that $P_V(\widehat{\mu}) = \widehat{\mu}_V$ and $P_W(\widehat{\mu}) = \widehat{\mu}_W$

\Return $\widehat{\mu}$
\EndIf
\EndFunction

\State

\Function{HuberOutlierGradientTruncation($S,\epsilon,p,\delta$)}{}
\If{$p=1$} 
\State Let $[a,b]$ be the smallest interval containing an $\left(1 - \epsilon - C \sqrt{ \frac{\log(|S|/\delta)}{|S|}} \right)(1 - \epsilon)$ fraction of points
\State $\widetilde{S} \leftarrow S \cap [a,b]$
\State \Return $\widetilde{S}$
\Else
\State Let $[S]_i$ be the samples with the $i^{\text{th}}$ coordinates only,  $[S]_i = \{ \left\langle x , e_i \right\rangle | x \in S \}$
\For{$i=1$ to $p$} 
\State $a[i] = \text{{\sc HuberEstimator}}([S]_i,\epsilon,1,\delta/p)$
%\State where $\Pi_{\Theta}(x) = \argmin_{y \in \Theta} \norm{y - x }{2}^2$
\EndFor
\State Let $B(r, a)$ be the ball of smallest radius centered at $a$  containing an \small{$(1 - \epsilon - C_p \left(\sqrt{\frac{p}{|S|} \log \left( \frac{|S|}{p \delta} \right)} \right) (1 - \epsilon)$} \normalsize { fraction of points in $S$
\State $\widetilde{S} \leftarrow S \cap B(r,a)$
\State \Return $\widetilde{S}$}
\EndIf
\EndFunction

\end{algorithmic}
\end{algorithm}

In terms of computational complexity, the initial screening step is coordinate-wise, hence $O(p)$. The dominant computation is to perform PCA (which has $O(p^3)$ complexity~\cite{pan1999complexity, johnstone2009sparse}) $\log_2(p)$ times. Thus, the overall runtime of Algorithm~\ref{alg:Huber estimator} is $\widetilde{O}(p^3)$.

%%%%%

\subsection{Heavy-Tailed Model}
\label{SecHT}

In the heavy-tailed model, we assume that data are drawn i.i.d.\ from a distribution with some number of finite moments. Note that the heavy-tailed model does not involve a contaminating distribution $Q$. However, the i.i.d.\ data may still appear to have ``outlier" points due to random sampling.

A popular approach for heavy-tailed mean estimation in the probably approximately correct (PAC) framework---obtaining high-probability deviation bounds which are as tight as possible under minimal distributional assumptions---is to use a  median-of-means (MOM) estimator. Roughly speaking, data are randomly partitioned into blocks, the mean of each block is computed, and the median of all of the block means is returned as the estimator. In multiple dimensions, different notions of medians exist, leading to different flavors of MOM estimators. For a more detailed overview, see the survey~\cite{lugosi2019mean} and the references cited therein. The MOM algorithm is summarized in Algorithm~\ref{alg:heavy tailed estimator}. In particular, we will employ a version of the algorithm from Minsker~\cite{minsker2015}, which combines the mean estimates using the geometric median, i.e., the point which minimizes the sum of $\ell_2$-distances to the block means.

\begin{algorithm}[h]
\centering
\caption{Heavy-Tailed  Estimator}\label{alg:heavy tailed estimator}
\begin{algorithmic}[1]
\Require Samples $S = \{s_i\}_{i = 1}^n$, Failure probability   $\delta$
\Function{HeavyTailedEstimator($S = \{s_i\}_{i = 1}^n$, $\delta$)}{}
\State Set $b = 1 + \lfloor 3.5\log{1/\delta}\rfloor$, the number of buckets

\State Partition $S$ into $b$ blocks $B_1, \dots, B_b$, each of size $\lfloor n/b \rfloor$

\For{$i = 1, \dots, b$}
 \State $\widehat{\mu}_i =  \frac{1}{|B_i|}\displaystyle\sum_{s \in B_i}s$
\EndFor
\State Set $\widehat{\mu} = \displaystyle \argmin_{\mu} \sum_{i = 1}^b\|\mu - \widehat{\mu}_i\|_2$

\Return $\widehat{\mu}$
\EndFunction
\end{algorithmic}
\end{algorithm}

Since the runtime of the geometric median computation on $n$ data points in $p$ dimensions is $\widetilde{O}(np)$~\cite{cohen2016geometric}, the runtime of Algorithm~\ref{alg:heavy tailed estimator} is $\widetilde{O}(n + bp) = \widetilde{O}(n+p)$.

%%%%%

\section{Robust Newton's Method}
\label{SecMain}

%\subsection{Overview}

We now present our variant of robust Newton's method. At each iterate, we will use gradient and Hessian estimates $(g(\theta), H(\theta))$ in place of $(\nabla \cR(\theta), \nabla^2 \cR(\theta))$ in the update equation~\eqref{EqnNewtonUpdate}. We assume that these estimates satisfy the conditions described in the following definitions:

\begin{definition}\label{defn: robust gradient estimator}(Prasad et al.~\cite{prasad2018})
A function $g(\theta)$ is a \emph{robust gradient estimator} for a data set $S = \{z_i\}_{i=1}^n$ if for functions $\alpha_g$ and $\beta_g$, with probability at least $1-\delta$, at any fixed $\theta \in \Theta$, the estimator satisfies the following inequality: 
\begin{equation}
\label{eqn:grad_estimator}
\|g(\theta) - \nabla \cR(\theta)\|_2 \leq \alpha_g(n, \delta)\|\theta - \theta^*\|_2 + \beta_g(n, \delta).
\end{equation}
%\textcolor{red}{Maybe we should call the gradient error $\alpha(n,\delta)$ and the Hessian error $\beta(n,\delta)$?}
\end{definition}

\begin{definition}\label{defn: robust Hessian estimator}
A function $H(\theta)$ is a \emph{robust Hessian estimator} for a data set $S = \{z_i\}_{i=1}^n$ if for functions $\alpha_h$ and $\beta_h$, with probability at least $1-\delta$, at any fixed $\theta \in \Theta$, the estimator satisfies the following inequality: 
\begin{equation}
\label{eqn:Hess_estimator}
\|H(\theta) - \nabla^2 \cR(\theta)\|_2 \leq \alpha_h(n, \delta)\|\theta - \theta^*\|_2 + \beta_h(n, \delta).
\end{equation}
\end{definition}

Successive iterates then take the form
\begin{equation*}
\theta_{t+1} = \theta_t - \alpha_t H(\theta_t)^{-1} g(\theta_t),
\end{equation*}
where $\alpha_t$ is chosen via a version of backtracking linesearch~\cite{boyd2004}. The exit condition of backtracking linesearch differs from its non-robust version in that function evaluations are replaced by robust estimates (cf.\ Lemmas~\ref{lem: rob est Hub} and~\ref{lem: rob est heavy} below) and an extra tolerance parameter $\zeta$ is included. The full algorithm is provided in Algorithm~\ref{alg:robust Newton}.

\begin{remark}
The statements of Definitions~\ref{defn: robust gradient estimator} and~\ref{defn: robust Hessian estimator} are written quite generally; in Section~\ref{SecRobustEst} below, we provide algorithms for obtaining robust gradient and Hessian estimators under both of our contamination models which can help elucidate the form of the bounds. See also Propositions~\ref{PropErrsHuber} and~\ref{PropErrsHeavy} in \ref{AppHuber} and~\ref{AppHeavy}, which provide explicit values of the parameters $(\alpha_g, \beta_g, \alpha_h, \beta_h)$ that are suitable for GLMs.
\end{remark}

%The ideas we use to deal with the fact that $\mathcal{R}(\theta)$ cannot be explicitly calculated are the following (here we assume that our loss function is $\mathcal{L}(\theta,z)= \frac{1}{2} (\theta^T x -y)^2 $, we also assume that the expectation of $x$ is zero):
%\begin{itemize}
%\item In order to estimate $\mathcal{R}(\theta)$, we calculate the median of the points $\{ \mathcal{L}(\theta,z_i) \}_{i=1}^n$.
%\item We estimate $\nabla \mathcal{R}(\theta) = \mathbb{E}[\nabla \mathcal{L}(\theta,z)]$ via using the algorithm $AgnosticMean$ on the data points $\{\nabla \mathcal{L}(\theta,z_i) \}_{i=1}^n$.
%\item Lastly for the estimation of $\nabla^{2} \mathcal{R}(\theta) = \mathbb{E}[\nabla^{2} \mathcal{L}(\theta,z)]$ we use the $CovarianceEstimation$ algorithm on the data points $\{x_i\}_{i=1}^n$ because $\nabla^{2} \mathcal{L}(\theta,z)=x x^T$ and hence $\nabla^{2} \mathcal{R}(\theta) = \mathbb{E}[ x x^T] = Cov(x)$.
%\item Alternatively we estimate $\nabla^{2} \mathcal{R}(\theta) = \mathbb{E}[\nabla^{2} \mathcal{L}(\theta,z)]$ via $HessianEstimation$ algorithm on the vectorized data points $\nabla^{2} \mathcal{L}(\theta,z)$ which implements the $AgnosticMean$ algorithm.
%\end{itemize}

%%%%%

\begin{algorithm}[h]
\centering
\caption{Robust Newton's Method}\label{alg:robust Newton}
\begin{algorithmic}[1]
\Require Data samples $S = \{z_i\}_{i = 1}^n$, Number of iterations $T$, Initial guess $\theta_0\in \Theta$, Backtracking linesearch parameters $\kappa_1 \in (0, 0.5), \kappa_2 \in (0,1)$, and $\zeta$
%\Require (Only if $\text{Type} = \text{Huber}$) Corruption Level $\epsilon$, Dimension $p$
\Function{RobustNewton($S, \theta_0, , \kappa_1, \kappa_2, \zeta$)}{}

\For{$t=0$ to $T-1$}

\State Compute losses  $ \{ \mathcal{L}(\theta_t,z_i) \}_{i=1}^n$  and gradients $ \{ \nabla \mathcal{L}(\theta_t,z_i) \}_{i=1}^n$
\State Compute gradient estimate $g(\theta_t) = \text{{\sc RobustGradientEstimate}}(S, \theta_t)$
\State Compute Hessian estimate $H(\theta_t) = \text{{\sc RobustHessianEstimate}}(S, \theta_t)$
\State Compute Newton step $\Delta \theta_{nt} = -H(\theta_t)^{-1} g(\theta_t)$
\State Compute stepsize $\alpha = \text{{\sc BacktrackingLineSearch}}(S, \theta_t, \Delta \theta_{nt}, g(\theta_t), \kappa_1, \kappa_2, \zeta)$
\State Update $\theta_{t+1} = \theta_t +  \alpha \Delta \theta_{nt}$

\EndFor

\Return $\theta_T$
\EndFunction

\State

\Function{BacktrackingLineSearch($S, \theta, \Delta \theta_{nt}, g(\theta), \kappa_1, \kappa_2, \zeta$)}{}

\State Set $\alpha = 1$
\While {$\text{{\sc RobustEstimate}} ( \{ \mathcal{L}(\theta +\alpha \Delta \theta_{nt}, z_i) \}_{i=1}^n) >  \text{{\sc RobustEstimate}} ( \{ \mathcal{L}(\theta, z_i) \}_{i=1}^n) + \kappa_1 \alpha g(\theta) \Delta\theta_{nt} + \zeta$}
\State Update $\alpha = \kappa_2 \alpha$
\EndWhile

\Return $\alpha$

\EndFunction

\end{algorithmic}
\end{algorithm}

\subsection{General Analysis for Robust Newton's Method}
\label{SecGeneral}

For the results of this section, we assume that $f(\theta) := \cR(\theta)$ is twice-differentiable and satisfies the Lipschitz condition $\| \D^2 f(\theta_1) -\D^2 f(\theta_2)\|_2 \leq L \|\theta_1-\theta_2\|_2$, for all $\theta_1, \theta_2$. We also assume that $f$ satisfies the strong convexity and smoothness conditions $m I \preceq \D^2 f(\theta) \preceq M I$, for all $\theta$ close enough to the initialization $\theta_0$. (We will verify these conditions for GLMs in Propositions~\ref{prop: L, m, M} and~\ref{PropConvexity} below.) Finally, we will assume that at each iterate, the gradient and Hessian estimates $g(\theta)$ and $H(\theta)$ satisfy inequalities~\eqref{eqn:grad_estimator} and~\eqref{eqn:Hess_estimator}, respectively. As demonstrated in Theorems~\ref{ThmGLMHuber} and~\ref{ThmGLMHeavy} later, the last condition can typically be justified w.h.p.\ via a union bound. Observe that in this setting, the unique global minimum of $f$ is the true parameter $\theta^*$.

%We make the following assumptions:
%$\Sigma_\theta = Cov(\D \mathcal{L}(\theta,z))$ finite,
%$f$ strongly convex,
%$m I \preceq \D^2 f(\theta) \preceq M I$ for all $\theta$,
%$\| \D^2 f(x) -\D^2 f(y)\|_2 \leq L \|x-y\|_2$ for all $x,y$. 
%Also assume that our gradient and Hessian estimates $g(\theta)$ and $H(\theta)$ satisfy inequalities~\eqref{eqn:grad_estimator} and~\eqref{eqn:Hess_estimator}, respectively.
%\textcolor{red}{I suppose we want to replace $f$ by $\cR$. Also, we are assuming strong convexity for all $\theta$. But probably we just want to assume convexity for all $\theta$, and strong convexity in some ball around $\theta^*$?}

The first result shows that if $\|\nabla f(\theta_0)\|_2$ is sufficiently small, the backtracking linesearch procedure will always choose stepsize 1. (This is known as the ``pure Newton" phase.) Furthermore, successive iterates converge at a geometric rate to a small ball around $\theta^*$. Recall that the parameters $(\kappa_1, \kappa_2)$ of backtracking linesearch are defined as in Algorithm~\ref{alg:robust Newton}.
%In the following result, let $\thetahat$ denote the minimizer of $f(\theta)$. 

\begin{theorem}
\label{ThmPure}
%the following inequalities for some constant multiples of $\se$: $e,E >0$:
%$\| g(\theta) - \D f(\theta) \|_2 \leq e $ and
%$ \| h(\theta) - \D^2 f(\theta) \|_2 \leq E $.
%Then we have that the estimates $\theta_t$ (output) of the algorithm Robust Newton Method (with input $z_1,z_2,...,z_n \sim (1-\epsilon)P+\epsilon Q$ data points), with starting point $\theta_0$ (such that $ \| \D f(\theta_0) \|_2 < \frac{m^2}{8L}$) satisfy the following, for $d_2$ a constant mutliple of $(e+E)$ (dependent on $m$, $M$ and $L$) and : 
Suppose $\|\nabla f(\theta_0)\|_2 < \eta$, where
\begin{equation}
\label{EqnEta}
\eta \defn \frac{m^2}{8L} \cdot \min\left\{3(1-2\kappa_1), 2 \right\}.
%\min\left\{\frac{3m(1-2\kappa_1)}{4L}, \frac{m^2}{4L}\right\}.
\end{equation}
Suppose the gradient and Hessian errors satisfy the bounds
\begin{equation}
\label{EqnGammaBd}
\gamma_g := \frac{2\eta \alpha_g}{m} + \beta_g \le \eta, \quad \text{and} \quad \gamma_h := \frac{2\eta \alpha_h}{m} + \beta_h \le \frac{m}{2}.
\end{equation}
Also suppose the robust estimates satisfy
\begin{equation}
\label{EqnMedianErr}
\left|\text{{\sc RobustEstimate}} ( \{ \mathcal{L}(\theta_t +\alpha \Delta \theta_t, z_i) \}_{i=1}^n) - f(\theta_t +\alpha \Delta \theta_t)\right| \le \frac{\zeta}{4},
\end{equation}
for each evaluation of backtracking linesearch, where we set the linesearch parameter to be 
\begin{align}\label{eqn: backtracking zeta}
    \zeta \ge \frac{8\gamma_g \eta}{m} + \frac{16 \gamma_h \eta^2}{m^2}.
\end{align}
%= 2\left(\frac{\gamma_g \cdot 2\eta}{m/2} + \frac{\gamma_h}{m} \cdot \frac{2\eta}{\sqrt{m/2}}\right)$
Then backtracking linesearch chooses unit steps on all successive iterates, and we have $\|\nabla f(\theta_t)\|_2 < \eta$ and
\begin{equation}
\label{EqnContract}
\|\theta_t - \theta^*\|_2 \le \frac{m}{L} \left(\frac{1}{2}\right)^{2^t} + \frac{6c_2}{m}
\end{equation}
%\[ \frac{m}{2} \| \theta_t - \theta^* \|_2 \leq
%\frac{m^3}{8L^2} \left(\frac{1}{2}\right)^{2^{t+1}}
%+\frac{m^3}{8L^2}d_2^2 
%+\frac{m^3}{4L^2} \left(\frac{1}{2}\right)^{2^t} d_2 \]
for all $t \ge 1$, where
\begin{equation*}
c_2 = \eta \left(\frac{4\gamma_g L}{m^2} + \frac{2\gamma_h}{m}\right) + \frac{2L\gamma_g^2}{m^2} + \gamma_g + \frac{2\gamma_g \gamma_h}{m},
\end{equation*}
and $(\gamma_g, \gamma_h)$ are small enough so that 
\begin{align}\label{eq: thmpure asmp c2}
    c_2 \le \min\left\{\frac{\eta}{2}, \frac{m^2}{24L}\right\}.
\end{align}
\end{theorem}

The proof of Theorem~\ref{ThmPure} is found in \ref{AppThmPure}.

Next, we show that after a finite number of steps, the iterates will indeed satisfy $\|\nabla f(\theta_t)\|_2 < \eta$, for an appropriate $\eta$. We can then apply Theorem~\ref{ThmPure} to the first iterate satisfying this condition, relabeling it as $\theta_0$, to obtain estimation error bounds on the overall trajectory of robust Newton's method. The proof of the following result is provided in \ref{AppThmDamped}.

\begin{theorem}
\label{ThmDamped}
Suppose the parameters $\eta$ and $\zeta$ are as defined in Theorem~\ref{ThmPure}. Define
\begin{equation*}
\gamma := \kappa_1 \kappa_2 \frac{m}{M} \left(\frac{1}{2} - \kappa_1\right) \frac{\eta^2}{4\sqrt{2M}}.
\end{equation*}
Suppose $(\gamma_g, \gamma_h)$ are chosen small enough such that $\zeta \le \frac{\gamma}{2}$ and conditions~\eqref{EqnGammaBd} and~\eqref{EqnMedianErr} are satisfied. Also suppose
\begin{equation}\label{eq: thmdamped asmp1}
\frac{2\alpha_g}{m} \sqrt{2M\left(f(\theta_0) - f(\theta^*)\right)} + \beta_g \le \min\left\{\frac{\eta}{2},  \sqrt{\frac{m}{2}} \cdot \frac{1}{2} \sqrt{\frac{\eta^2}{4\sqrt{2M}}}\right\},
\end{equation}
and
\begin{equation}\label{eq: thmdamped asmp2}
\frac{2\alpha_h}{m} \sqrt{2M\left(f(\theta_0) - f(\theta^*)\right)} + \beta_h \le M.
\end{equation}
If $\|\nabla f(\theta_t)\|_2 \ge \eta$ for an iterate $t \ge 0$, then $f(\theta_t) \le f(\theta_0)$ and $f(\theta_{t+1}) - f(\theta_t) < - \frac{\gamma}{2}$.
\end{theorem}

The preceding theorem directly implies that after a finite number of steps (known as the ``damped Newton" phase), all successive iterates of the algorithm satisfy $\|\nabla f(\theta_t)\|_2 < \eta$.
Indeed, Theorem~\ref{ThmDamped} guarantees that $f(\theta_{t+1}) - f(\theta_t) < -\frac{\gamma}{2}$ whenever $\|\nabla f(\theta_t)\|_2 \ge \eta$, where $\gamma > 0$. Since $\theta^* \in \arg\min_\theta f(\theta)$, we clearly could not have $\|f(\theta_t)\|_2 \ge \eta$ for all $0 \le t \le T$, where $T = \lceil \frac{f(\theta_0) - f(\theta^*)}{\gamma/2}\rceil$, or else
\begin{equation*}
f(\theta_T) - f(\theta^*) = \left(f(\theta_0) - f(\theta^*)\right) + \sum_{t=0}^{T-1} \left(f(\theta_{t+1}) - f(\theta_t)\right) < \left(f(\theta_0) - f(\theta^*)\right) - \frac{T\gamma}{2} < 0,
\end{equation*}
contradicting the minimality of $\theta^*$.

%%%%%

\subsection{Robust Estimation of Gradients and Hessians}
\label{SecRobustEst}

In this subsection, we explain how robust estimators for gradients and Hessians can be obtained under two models of contamination, namely the Huber $\epsilon$-contamination model and the heavy-tailed model. 
%We begin with the specific cases of linear regression and GLMs, for which we have provably robust estimators. We then give an algorithm for robust Hessian estimation for a general data model.

\subsubsection{Robust Gradient Estimation}

For the $\epsilon$-contamination model, we obtain a robust gradient estimate by applying Algorithm~\ref{alg:Huber estimator} to the gradients computed on each of the $n$ sampled data points. Similarly, for the heavy-tailed model, we use Algorithm~\ref{alg:heavy tailed estimator} to obtain a robust gradient estimate. For completeness, we summarize this procedure in Algorithm~\ref{alg: robust gradient estimator}.

\begin{algorithm}[H]
\centering
\caption{Robust Gradient Estimator}\label{alg: robust gradient estimator}
\begin{algorithmic}[1]
\Require Samples $S = \{z_i\}_{i = 1}^n$, Parameter $\theta$, Contamination type $\text{{\sc Type} }$
\Require (If $\text{{\sc Type}} = \textnormal{Huber}$) Corruption Level $\epsilon$, Dimension $p$, Failure probability $\delta$
\Require (If $\text{{\sc Type}} = \textnormal{Heavy-tail}$) Failure probability $\delta$
\Function{RobustGradientEstimator($S, \theta, \text{{\sc Type}}, \epsilon, p, \delta $)}{}

\State Compute $\{\D\cL(\theta, z_i)\}_{i=1}^n$, the gradient of the loss at each data point in $S$

\If{$\text{{\sc Type}} = \textnormal{Huber}$}

\Return $\text{{\sc HuberEstimator}}( \{\D\cL(\theta, z_i)\}_{i=1}^n,\epsilon,p,\delta)$

\EndIf

\If{$\text{{\sc Type}} = \textnormal{Heavy-tail}$}

\Return $\text{{\sc HeavyTailedEstimator}}( \{\D\cL(\theta, z_i)\}_{i=1}^n,\delta)$

\EndIf

\EndFunction
\end{algorithmic}
\end{algorithm}

The following lemmas, borrowed from Prasad et al.~\cite{prasad2018}, show that Algorithm~\ref{alg: robust gradient estimator} returns a robust gradient estimator that satisfies Definition~\ref{defn: robust gradient estimator}.

\begin{lemma*}
[Lemma 1 of Prasad et al.~\cite{prasad2018}]\label{lem: grad est huber}
Let  $\{z_i\}_{i=1}^n$ be $n$ i.i.d.\ samples drawn from a Huber $\epsilon$-contaminated distribution~\eqref{eq: huber model}. Let the true distribution of gradients $\D\cL(\theta, z)$, with $z$ drawn from $P$, have bounded fourth moments. Then Algorithm~\ref{alg: robust gradient estimator} with $S=\{z_i\}_{i=1}^n$, $\text{{\sc Type}} = \textnormal{Huber}$, and any $\theta\in \Theta$ returns a gradient estimate $g(\theta)$ that satisfies
\begin{align}\label{eq: lem grad est huber}
    \| g(\theta) - \E[\D \cL(\theta,z)] \|_2 \leq C_1 (\se + \gamma(n,p,\delta,\epsilon)) \sqrt{ \|\Cov(\D \cL(\theta,z)) \|_2 \log p},
\end{align}
with probability at least $1-\delta$, where $C_1 > 0$ is a constant and $\gamma$ is given by
\begin{align}\label{eq: gamma}
    \gamma(n,p,\delta,\epsilon) = \left( \frac{p \log(p) \log(n/(p\delta))}{n}\right)^{3/8} + \left(\frac{\epsilon p^2 \log(p) \log(p \log(p)/\delta)}{n}\right)^{1/4}.
\end{align}
\end{lemma*}

\begin{lemma*}
[Lemma 2 of Prasad et al.~\cite{prasad2018}]\label{lem: grad est heavy-tail}
Let  $\{z_i\}_{i=1}^n$ be $n$ i.i.d.\ samples drawn from a heavy-tailed distribution $P$ such that the true distribution of gradients $\D\cL(\theta, z)$  has bounded second moments. Then Algorithm~\ref{alg: robust gradient estimator} with $S=\{z_i\}_{i=1}^n$, $\text{{\sc Type}} = \textnormal{Heavy-tail}$, and any $\theta\in \Theta$ returns a gradient estimate $g(\theta)$ that satisfies
\begin{align}\label{eq: lem grad est heavy tail}
    \| g(\theta) - \E[\D \cL(\theta,z)] \|_2 \leq 11\sqrt{ \frac{\tr(\Cov(\D \cL(\theta,z)))\log(1.4/\delta)}{n}},
\end{align}
with probability at least $1-\delta$.
\end{lemma*}

%%%%%

\subsubsection{Robust Hessian Estimation: The Vectorizing Approach }

The procedure for obtaining a robust Hessian, summarized in Algorithm~\ref{alg: robust Hessian estimator}, is similar to that of Algorithm~\ref{alg: robust gradient estimator}, except that the appropriate multivariate estimation procedure is applied to a vectorized version of the Hessian matrix (where we use $\flatten(A)$ to denote a vectorized version of the matrix $A$, and use $\unflatten()$ to denote the inverse function).

\begin{algorithm}[H]
\centering
\caption{Robust Hessian Estimator}\label{alg: robust Hessian estimator}
\begin{algorithmic}[1]
\Require Samples $S = \{z_i\}_{i = 1}^n$, Parameter $\theta$, Contamination type $\text{{\sc Type} }$
\Require (If $\text{{\sc Type}} = \textnormal{Huber}$) Corruption Level $\epsilon$, Dimension $p$, Failure probability $\delta$
\Require (If $\text{{\sc Type}} = \textnormal{Heavy-tail}$) Failure probability $\delta$
\Function{RobustHessianEstimator($S, \theta, \text{{\sc Type}}, \epsilon, p, \delta $)}{}

\State Compute $\{\D^2\cL(\theta, z_i)\}_{i=1}^n$, the Hessian of the loss at each data point in $S$

\If{$\text{{\sc Type}} = \textnormal{Huber}$}

\Return $\unflatten(\text{{\sc HuberEstimator}}( \{\flatten(\D^2\cL(\theta, z_i))\}_{i=1}^n,\epsilon,p,\delta))$

\EndIf

\If{$\text{{\sc Type}} = \textnormal{Heavy-tail}$}

\Return $\unflatten(\text{{\sc HeavyTailedEstimator}}( \{\flatten(\D^2\cL(\theta, z_i))\}_{i=1}^n,\delta))$

\EndIf

\EndFunction
\end{algorithmic}
\end{algorithm}

%\textcolor{red}{Should the proofs of the following lemmas just follow Lemmas~\ref{lem: grad est huber} and~\ref{lem: grad est heavy-tail}?}
%\textcolor{blue}{True. Should we remove them?}

The next two lemmas follow immediately from the arguments used to derive Lemmas~\ref{lem: grad est huber} and~\ref{lem: grad est heavy-tail}:

\begin{lemma*}
\label{lem: Hessian est huber}
Let  $\{z_i\}_{i=1}^n$ be $n$ i.i.d.\ samples drawn from a Huber $\epsilon$-contaminated distribution~\eqref{eq: huber model}.
Suppose $\Cov(\flatten(\D^2 \cL(\theta,z)))$ is finite and  $\flatten(\D^2\cL(\theta, z)))$ has bounded fourth moments.
Then Algorithm~\ref{alg: robust gradient estimator} with $S=\{z_i\}_{i=1}^n$, $\text{{\sc Type}} = \textnormal{Huber}$,  and any $\theta\in \Theta$ returns a Hessian estimate $H(\theta)$ that satisfies
\begin{align}
    \| H(\theta) - \E[\D^2 \cL(\theta,z)]  \|_2 \leq C_2 (\se + \gamma(n,p,\delta,\epsilon))\sqrt{\|\Cov(\flatten(\D^2\cL(\theta, z))) \|_2 \log p},
\end{align}
with probability at least $1-\delta$, where $C_2 > 0$ is a constant and $\gamma$ is given by equation~\eqref{eq: gamma}.
\end{lemma*}

%\begin{proof}
%The desired result follows by applying Lemma 1 in \cite{prasad2018} on $\flatten(\D^2\cL(\theta, z)))$.
%\end{proof}

\begin{lemma*}
\label{lem: Hessian est heavy-tail}
Let  $\{z_i\}_{i=1}^n$ be $n$ i.i.d.\ samples drawn from a heavy-tailed distribution $P$.
Suppose $\Cov(\flatten(\D^2 \cL(\theta,z)))$ is finite.
Then Algorithm~\ref{alg: robust gradient estimator} with $S=\{z_i\}_{i=1}^n$, $\text{{\sc Type}} = \textnormal{Heavy-tail}$, and any $\theta\in \Theta$ returns a Hessian estimate $H(\theta)$ that satisfies
\begin{align}
    \| H(\theta) - \E[\D^2 \cL(\theta,z)] \|_2 \leq  C_3\sqrt{ \frac{\tr(\Cov(\flatten(\D^2 \cL(\theta,z))))\log(1.4/\delta)}{n}},
\end{align}
with probability at least $1-\delta$, where $C_3>0$ is a constant.
\end{lemma*}

%\begin{proof}
%The desired result follows by applying Corollary 4.1 in Minsker~\cite{minsker2015} on $\flatten(\D^2\cL(\theta, z)))$.
%\end{proof}

%%%%%

%%%%%

\section{Application to GLMs}
\label{SecApps}

In this section, we apply the robust Newton method to parametric estimation in GLMs. We consider the Huber $\epsilon$-contamination model in Section~\ref{sec: app huber}, and we consider the heavy-tailed contamination model in Section~\ref{sec: app heavy-tail}.

Throughout this section, we will assume that the uncontaminated model is a GLM of the form~\eqref{eq: model GLM}.
%We further assume that the covariates $x\in \real^p$ have bounded fourth moment, mean $0$, and covariance matrix $\Sigma_x$.
Consider the loss function in equation~\eqref{eq:loss GLM}.
We assume that the link function $\Phi$ of the GLM satisfies the following bounds:
\begin{align}
    \E\left[ |\Phi'(x_i^T\theta) - \Phi'(x_i^T\theta^*)|^{2k} \right] 
    &\leq L_{\Phi,2k} \|\theta-\theta^*\|_2^2 + B_{\Phi, 2k}, \qquad \forall \theta \in \Theta, \label{eq: phi' assump}\\
    \E\left[ |\Phi^{(t)}(x_i^T\theta^* )|^k \right] &\leq M_{\Phi, t, k}, \label{eq: phi^(t) assump}
\end{align}
and
\begin{align}\label{eq: phi^(t) assump bounded}
    \|\Phi^{(t)}\|_\infty \leq \overline{M}_{\Phi, t},
\end{align}
for pairs $(k,t)$ to be specified in the sequel, where $\Phi^{(t)}$ is the $t^{\text{th}}$ derivative of $\Phi$. 

We also make assumptions on the boundedness of moments of $x_i \in \real^p$. We say that $x_i$ has bounded $2k^{\text{th}}$ moments if there is a constant $\widetilde{C}_{2k}>0$ such that for every unit vector $v\in \real^d$, we have $\E[(x_i^Tv)^{2k}]\leq \widetilde{C}_{2k}\left(\E[(x_i^Tv)^{2}] \right)^k$.

\begin{assumption}
\label{AssCov}
Suppose the distribution of the $x_i$'s has bounded eighth moments. Let $\Sigma_x$ denote the finite covariance matrix of the $x_i$'s.
\end{assumption}

In order to apply Theorems~\ref{ThmPure} and \ref{ThmDamped} to GLMs, we need the Hessian $\nabla^2 \cR(\theta)$ to be Lipschitz smooth and satisfy $mI \preceq \nabla^2 \cR(\theta) \preceq MI$ for all $\theta\in \Theta$  close enough to the initialization $\theta_0$. We now verify these assumptions. The following results are proved in \ref{AppPropLM} and~\ref{AppPropConvexity}.

\begin{proposition}\label{prop: L, m, M}
Let the link function $\Phi$ satisfy inequality~\eqref{eq: phi^(t) assump bounded} for $t=3$, and suppose Assumption~\ref{AssCov} is satisfied. Then the Hessian $\nabla^2 \cR(\theta)$ is $L$-Lipschitz and satisfies $ \nabla^2 \cR(\theta) \preceq MI$ with
\begin{equation}
\label{EqnLM}
L \defn \sqrt{\widetilde{C}_4 \overline{M}_{\Phi, 3}}\|\Sigma_x\|_2, \qquad M \defn  \overline{M}_{\Phi, 2}\sqrt{\widetilde{C}_4}\|\Sigma_x\|_2.
\end{equation}
\end{proposition}

\begin{proposition}
\label{PropConvexity}
Suppose there exist constants $B, \tau > 0$ such that for any $\theta \in \real^p$ such that $\|\theta\|_2 \le B$, we have
\begin{equation}
\label{EqnXCond}
\widetilde{C}_4 \|\Sigma_x\|_2^2 \cdot \mprob(|x_i^T \theta| > \tau) \le \frac{1}{4} \lambda_{\min}^2(\Sigma_x).
\end{equation}
Define $b_\tau := \inf_{|u| \le \tau} \Phi''(u)$. Then $\frac{b_\tau}{2} \lambda_{\min}(\Sigma_x) I \preceq \nabla^2 \cR(\theta)$ for all $\theta \in \real^p$ such that $\|\theta\|_2 \le B$.
\end{proposition}

\begin{remark}
\label{RemBounded}
Note that when the covariates are sub-Gaussian, we can certainly guarantee that the tail condition~\eqref{EqnXCond} is satisfied for sufficiently large $\tau$, since $x_i^T \theta$ is sub-Gaussian with parameter scaling with $B$ and the sub-Gaussian parameter $\sigma_x^2$ of $x_i$. Thus, we have
\begin{equation*}
\mprob(|x_i^T\theta| > \tau) \le c_1 \exp\left(-\frac{c_2\tau^2}{B^2 \sigma_x^2}\right),
\end{equation*}
and it suffices to take $\tau = c_3 B \sigma_x \log^{1/2}\left(\frac{c_4 \|\Sigma_x\|_2^2}{\lambda^2_{\min}(\Sigma_x)}\right)$. Furthermore, in the proofs of Theorems~\ref{ThmPure} and~\ref{ThmDamped} (cf.\ inequalities~\eqref{EqnOptErr} and~\eqref{EqnChain}, respectively), we show that $\|\theta_t - \theta^*\|_2$ remains bounded (where the bound depends on $\theta_0$ and the problem parameters).

In the case of logistic regression, we have $\Phi''(u) = \frac{e^u}{(1+e^u)^2}$, and it is easy to see that $b_\tau > 0$ for any value of $\tau$.
%The Hessian $\nabla^2 \cR(\theta) = \E_x\left[ xx^T \left( \Phi''(x^T\theta)\right) \right]$ satisfies  $mI \preceq \nabla^2 \cR(\theta)$ if the covariance matrix $\Sigma_x$ satisfies $m'\preceq \Sigma_x$ and the second derivative of link function is bounded below. For example, in the case of logistic regression,
%\begin{align*}
%    \nabla^2 \cR(\theta)
%    &= \E_x\left[ xx^T \frac{e^{x^T\theta}}{\left( 1+e^{x^T\theta} \right)^2} \right].
%\end{align*}
%For all $\theta$ in a bounded subset of $\Theta$ that is sufficiently close to the initialization $\theta_0$, the above Hessian satisfies $mI \preceq \nabla^2 \cR(\theta)$ as long as $\Sigma_x$ is non-singular.
\end{remark}

\subsection{Preliminary Error Bounds}

From Lemmas~\ref{lem: grad est huber} and~\ref{lem: grad est heavy-tail}, we see that the term $\Cov(\D \cL(\theta,z))$ plays a crucial role in proving that our gradient estimates are robust. Likewise, Lemmas~\ref{lem: Hessian est huber} and~\ref{lem: Hessian est heavy-tail} show the importance of the term $\Cov(\flatten(\D^2\cL(\theta, z)))$ in proving that the Hessian estimates are robust. The following two lemmas provide upper bounds on these two terms for the specific case of GLMs:

\begin{lemma*}
[Lemma 4 in Prasad et al.~\cite{prasad2018}]\label{lem: grad est GLM}
Let  $\{z_i\}_{i=1}^n$ be $n$ i.i.d.\ samples drawn from a distribution that satisfies the GLM model~\eqref{eq: model GLM}. Let the link function $\Phi$ satisfy inequalities~\eqref{eq: phi' assump} and~\eqref{eq: phi^(t) assump} for $k \in \{1,2\}$ and $t \in \{2, 4\}$, and suppose Assumption~\ref{AssCov} is satisfied. Then the true distribution of gradients $\D \cL(\theta,z)$ has bounded fourth moments. Moreover, 
\begin{align}
    \|\Cov(\D \cL(\theta,z))\|_2
    &\leq C_1 \|\Sigma_x\|_2 \left(\sqrt{L_{\Phi, 4}} + L_{\Phi, 2} \right) \|\theta-\theta^*\|_2^2 \nonumber\\
    &\ \ \ + C_2 \|\Sigma_x\|_2 \left(B_{\Phi, 2} + \sqrt{B_{\Phi, 4}} + c(\sigma) \sqrt{M_{\Phi, 2, 2}} + \sqrt{c(\sigma)^3 M_{\Phi, 4, 1}}\right),
\end{align}
where $C_1, C_2 > 0$ are constants.
\end{lemma*}

\begin{lemma*}
\label{lem: Hess est GLM}
Let  $\{z_i\}_{i=1}^n$ be $n$ i.i.d.\ samples drawn from a distribution that satisfies the GLM model~\eqref{eq: model GLM}. Let the link function $\Phi$ satisfy inequality~\eqref{eq: phi^(t) assump bounded} for $t=2$, and suppose Assumption~\ref{AssCov} is satisfied. Then the distribution of the flattened Hessian $\flatten(\D^2 \cL(\theta,z))$ has bounded fourth moments. Moreover, we have
%there exists a constant $C>0$ such that
\begin{align}
\label{EqnTrCov}
     \tr(\Cov(\flatten(\D^2 \cL(\theta,z)))) \leq \overline{M}_{\Phi, 2}^2 \widetilde{C}_4p^2 \|\Sigma_x\|_2^2.
\end{align}
\end{lemma*}

The proof of Lemma~\ref{lem: Hess est GLM} is contained in \ref{AppLemHessGLM}.

\begin{remark}
\label{Rem4Wise}
Note that under additional assumptions (e.g., 4-wise independence of the components of the $x_i$'s), we can prove that
\begin{equation*}
\|\Cov(\flatten(\D^2 \cL(\theta, z)))\|_2 \le C\widetilde{C}_4 \|\Sigma_x\|_2^2,
\end{equation*}
for some constant $C > 0$, which avoids an extra dimension-dependent factor in comparison to inequality~\eqref{EqnTrCov} (cf.\ Proposition 4.2 in Lai et al.~\cite{lai2016}) for the Huber contamination setting. Indeed, only the spectral norm of the covariance of the flattened Hessian appears in the deviation bound of Lemma~\ref{lem: Hessian est huber} (Huber's $\epsilon$-contamination model);  the trace of the covariance appears in Lemma~\ref{lem: Hessian est heavy-tail} (heavy-tailed model).
\end{remark} 

%\textcolor{red}{Insert lemma (a direct corollary of the one to be included in Section 3) on medan concentration for GLMs on bounded domains?}

For applying Theorems~\ref{ThmPure} and \ref{ThmDamped}, we also need the robust estimate of the losses to be close to the population risk, as in inequality~\eqref{EqnMedianErr}. In the following two lemmas, we show that this assumption holds with high probability for the robust estimates obtained by applying Algorithms~\ref{alg:Huber estimator} and \ref{alg:heavy tailed estimator} on the losses. Further note that the following lemmas require boundedness of higher-order moments of $\cL(\theta, z)$, which can be justified in our scenario if $\theta$ is bounded. As mentioned in Remark~\ref{RemBounded}, we can indeed assume that the iterates $\{\theta_t\}$, to which Lemmas~\ref{lem: rob est Hub} and~\ref{lem: rob est heavy} are applied in the sequel, are bounded.

The following result is a consequence of Lemma 14 in Prasad et al.~\cite{prasad2018}:

\begin{lemma*}
\label{lem: rob est Hub}
Let  $\{z_i\}_{i=1}^n$ be $n$ i.i.d.\ samples drawn from a Huber $\epsilon$-contaminated distribution~\eqref{eq: huber model}, where the true distribution satisfies the GLM model~\eqref{eq: model GLM}. Let $\cL(\theta, z)$ have bounded fourth moments. Then with probability at least $1-\delta$, the robust estimate returned by Algorithm~\ref{alg:Huber estimator} satisfies
\begin{multline*}
    \left|\text{{\sc HuberEstimate}} ( \{ \mathcal{L}(\theta, z_i) \}_{i=1}^n) - \cR(\theta) \right| \\
    \leq C_1 \left( \epsilon + \sqrt{\frac{\log(n/\delta)}{n}} \right)^\frac{3}{4} + C_2 \left( \epsilon +  \sqrt{\frac{\log(n/\delta)}{n}} \right)^\frac{1}{2} \frac{\log(1/\delta)}{n},
\end{multline*}
where $C_1, C_2 > 0$ are constants.
\end{lemma*}

%%%%%

\begin{comment}
\begin{lemma}
\label{lem: rob est heavy tail}
Let $\var_z(\cL(\theta, z)) \leq C_{\Theta}$ for all $\theta$ satisfying $\|\theta-\theta_0\|_2\leq B_{\Theta}$. Then for all $\theta$ satisfying $\|\theta-\theta_0\|_2\leq B_{\Theta}$, we have  the following with probability at least $1-\delta$.
\begin{align}
    |\med(\{\cL(\theta, z_i)\}_{i=1}^n) - \E_z[\cL(\theta, z)]| 
    \leq 
    \sqrt{\frac{C_\Theta}{n}} + 
    \frac{C_\Theta}{n\sqrt{\delta}}.
\end{align}
\end{lemma}

\begin{proof}
By the mean-median inequality (see for example, Exercise 2.1 in \cite{Bou13}) we have,
\begin{align*}
    \left| \med(\{\cL(\theta, z_i)\}_{i=1}^n) - \frac{1}{n}\sum_{i=1}^n \cL(\theta, z_i) \right| \leq \sqrt{\frac{\var_z(\cL(\theta, z))}{n}}.
\end{align*}
By Chebyshev's inequality, with probability at least $1-\delta$, we have
\begin{align*}
    \left| \frac{1}{n}\sum_{i=1}^n \cL(\theta, z_i) - \cR(\theta) \right|
    \leq \frac{\var_z(\cL(\theta, z))}{n\sqrt{\delta}}.
\end{align*}
Combining the above two inequalities, we get the desired result.
\end{proof}
\end{comment}

The next result follows from similar arguments to those in Lemma~\ref{lem: grad est heavy-tail}:

\begin{lemma*}
\label{lem: rob est heavy}
Let  $\{z_i\}_{i=1}^n$ be $n$ i.i.d.\ samples drawn from a heavy-tailed distribution $P$ that satisfies the  GLM model~\eqref{eq: model GLM}.
Let $\cL(\theta, z)$ have bounded second moments. Then with probability at least $1-\delta$, the robust estimate returned by Algorithm~\ref{alg:heavy tailed estimator} satisfies
\begin{equation*}
\left|\text{{\sc HeavyTailedEstimate}} ( \{ \mathcal{L}(\theta, z_i) \}_{i=1}^n) - \cR(\theta) \right|
\le C \sqrt{\frac{\log\left(\frac{1.4}{\delta}\right)}{n}},
\end{equation*}
with probability at least $1-\delta$, where $C > 0$ is a constant.
\end{lemma*}

%%%%%

\subsection{Huber Contamination}
\label{sec: app huber}

Throughout this subsection, we work under the following assumptions:
\begin{assumption}
\label{AssHuber}
Suppose the link function $\Phi$ satisfies inequalities~\eqref{eq: phi' assump} and~\eqref{eq: phi^(t) assump} for $k \in \{1,2,4\}$ and $t \in \{2, 4\}$, and inequality~\eqref{eq: phi^(t) assump bounded} for $t \in \{2,3\}$.
Also suppose $\cR(\theta)$ is $m$-strongly convex, i.e., $mI \preceq \nabla^2 \cR(\theta)$ uniformly over $\theta$, and $L$ and $M$ are defined as in equation~\eqref{EqnLM}.
\end{assumption}

%\textcolor{red}{Should we also analyze the sample median here? Note that the proofs of Theorems~\ref{ThmPure} and~\ref{ThmDamped} (cf.\ inequalities~\eqref{EqnOptErr} and~\eqref{EqnChain}, respectively) show that $\|\theta_t - \theta^*\|_2$ remains bounded (where the bound depends on $\theta_0$ and the problem parameters). Then the error bounds on the median estimates will follow from Lemma~\ref{LemMedian}.}

We then have the following result, proved in \ref{AppThmGLMHuber}:

\begin{theorem}
\label{ThmGLMHuber}
Let  $\{z_i\}_{i=1}^n$ be $n$ i.i.d.\ samples drawn from a Huber $\epsilon$-contaminated distribution~\eqref{eq: huber model}, where the true distribution satisfies the GLM model~\eqref{eq: model GLM} and the conditions of Assumptions~\ref{AssCov} and~\ref{AssHuber} are satisfied.
Define $(\eta, c_2, \gamma_g, \gamma_h)$ as in Theorem~\ref{ThmPure} and $\gamma$ as in Theorem~\ref{ThmDamped}. Suppose
\begin{equation}
\label{EqnCondZeta}
\zeta = C\max\left\{\frac{\gamma_g \eta}{m} + \frac{\gamma_h \eta^2}{m^2}, \left(\epsilon + \sqrt{\frac{\log(n/\delta)}{n}} \right)^\frac{3}{4} + \left( \epsilon +  \sqrt{\frac{\log(n/\delta)}{n}} \right)^\frac{1}{2} \frac{\log(1/\delta)}{n}\right\} \le \frac{\gamma}{2}.
\end{equation}
Let $\delta>0$.
Define
\begin{align}
    \widehat{\gamma}
    &\defn  \min\left\{ \frac{\widehat{\alpha}_g}{c_1 \sqrt{\|\Sigma_x\|_2 \log p}}, \frac{\widehat{\beta}_g}{c_2 \sqrt{\|\Sigma_x\|_2 \log p}}, \frac{\widehat{\beta}_h}{c_3 \|\Sigma_x\|_2 p \sqrt{\log p}}  \right\},
\end{align}
where $(\widehat{\alpha}_g, \widehat{\beta}_h, \widehat{\beta}_g)$ are as defined in Lemma~\ref{lem: thm asmps}.
%\textcolor{red}{Remove the following conditions on $\gamma$, $\hat{\gamma}$, and $\epsilon$?}
Suppose $n$ and $\epsilon$ are such that
\begin{align}
\label{EqnCondGamma}
     \se + \gamma\left(n,p,\delta,\epsilon\right)
     < \widehat{\gamma}.
\end{align}
%\begin{equation}
%\label{EqnCondZeta}
%C_8 \left( \epsilon + \sqrt{\frac{\log(n/\delta)}{n}} \right)^\frac{3}{4} + C_9 \left( \epsilon +  \sqrt{\frac{\log(n/\delta)}{n}} \right)^\frac{1}{2} \frac{\log(1/\delta)}{n} \le \frac{\zeta}{4}.
%\end{equation}
Then applying Algorithm~\ref{alg:robust Newton} on $\{z_i\}_{i=1}^n$ with initialization $\theta_0\in \Theta$ and number of iterations
 \begin{align*}
    T \ge  \frac{\cR(\theta_0) - \cR(\theta^*)}{\gamma/2} + \log_2\log_2\left(\frac{6c_2L}{m^2}\right)
\end{align*}
returns an output such that 
\begin{equation*}
\|\theta_T - \theta^*\|_2 \leq \frac{12c_2}{m} =
%O\left(\epsilon\log(p) + \sqrt{\epsilon\log(p)}\right),
O\left(p \sqrt{\epsilon \log p}\right),
\end{equation*}
with probability at least $1 - T\delta\left(5 + \Big\lceil\frac{\log\left(\frac{m}{M}\left(\frac{1}{2} - \kappa_1\right)\right)}{\log \kappa_2}\Big\rceil \right)$. 
\end{theorem}

\begin{remark}
\label{RemThmCond}
Examining the condition~\eqref{EqnCondGamma}, we see that (assuming $(m, M, L, \eta, \|\Sigma_x\|_2)$ are all constants) we have a required upper bound of $\widehat{\gamma}^2 \asymp \frac{1}{\log p}$ on the contamination proportion $\epsilon$. Furthermore, from the expression~\eqref{eq: gamma}, we have (ignoring log factors) $n \succsim \max\left\{p, \epsilon p^2 \right\}$. The condition~\eqref{EqnCondZeta} likewise gives a minimum sample size requirement on $n$ in terms of $\delta$.
\end{remark}

\begin{remark}
It is instructive to compare the result of Theorem~\ref{ThmGLMHuber} to Theorem 4 in Prasad et al.~\cite{prasad2018}, which gives a convergence statement of the form
\begin{equation*}
\|\theta^t - \theta^*\|_2 \le \kappa^t \|\theta^0 - \theta^*\|_2 + \frac{C\|\Sigma_x\|_2^{1/2}\sqrt{\log p}}{1-\kappa} \left(\sqrt{\epsilon} + \gamma(n,p,\delta,\epsilon)\right)
\end{equation*}
for iterates $\{\theta^t\}_{t \ge 0}$ of robust gradient descent. For sufficiently large $t$, the second term dominates, leaving an error term of $O(\sqrt{\epsilon \log p})$. Our theorem has a dominant factor of $O(p \sqrt{\epsilon \log p})$, which can be reduced to $O(\sqrt{\epsilon \log p})$ if we assume 4-wise independence of the coordinates of the covariate distribution (cf.\ Remark~\ref{Rem4Wise} above). In terms of the convergence rate of the optimization procedure, however, we just need $T \asymp \log\log \frac{1}{\epsilon}$, compared to $T \asymp \log \frac{1}{\epsilon}$ in the case of robust gradient descent.
\end{remark}

Linear regression is of course a special case of GLMs, for which Theorem~\ref{ThmGLMHuber} readily applies. On the other hand, note that a much more direct way to obtain a robust estimator for linear regression would be to directly robustify the estimator $\thetahat_{OLS} = \left(\frac{X^TX}{n}\right)^{-1} \left(\frac{X^T y}{n}\right)$, where we apply Algorithm~\ref{alg:Huber estimator} to obtain robust estimates of $\E[y_i x_i]$ and $\E[x_i x_i^T]$ (the latter matrix being vectorized before applying the agnostic mean algorithm). Indeed, in the non-robust case, applying Newton's method to the ordinary least squares objective converges in a single step. A careful analysis of this so-called ``robust plug-in estimator" would also give an error of $O(\sqrt{\epsilon})$ in the robust case, but a direct analysis would provide an error bound which depends on $\|\theta^*\|_2$, since $\Cov(x_i, y_i)$ would scale with $\|\theta^*\|_2$ (cf.\ Corollary 3 in Prasad et al.~\cite{prasad2018}). On the other hand, the guarantee of Theorem~\ref{ThmGLMHuber} for the full robust Newton's method does not involve $\|\theta^*\|_2$.
%Our simulations in Section~\ref{SecSims} later (\textcolor{red}{cite?}) reveal that, unlike the robust plug-in estimator, the Newton's method-based robust estimator indeed continues to have small error even as $\|\theta^*\|_2 \rightarrow \infty$.

\begin{remark}
A natural question is whether the estimation error upper bounds in Theorem~\ref{ThmGLMHuber} are tight: For i.i.d.\ samples from a GLM with Huber $\epsilon$-contamination, is it possible to derive estimators with error smaller than $C\sqrt{\epsilon}$? For the case of linear regression, this problem has been studied quite carefully, and it has been established that when the uncontaminated data are Gaussian with an isotropic covariance, the rate should be $\Theta(\epsilon)$~\cite{chen2016general, diakonikolas2019efficient, pensia2020robust, depersin2020spectral}, with no dependence at all on $p$.
%\textcolor{red}{Check if this is still true for non-isotropic covariances.}
In the case when the covariates only follow a bounded fourth moment assumption, the rate improves to $\Theta(\sqrt{\epsilon})$~\cite{cherapanamjeri2020optimal, cherapanamjeri2020algorithms}. We are not aware of existing lower bounds in the literature for more general GLMs, though it is reasonable to conjecture that the optimal rates for estimation in the Huber contamination model can also be made dimension-independent.
\end{remark}

In terms of computational complexity, the overall complexity of the robust Newton method is the number of iterations $T$ multiplied by the computational complexity of robust gradient/Hessian computations. As mentioned at the end of Section~\ref{SecHubEps}, the runtime of Algorithm~\ref{alg:Huber estimator} is $\widetilde{O}(p^3)$; since we would be applying this to the vectorized Hessian matrices, the computational complexity of the robust Newton method would then be $\widetilde{O}(Tp^6)$ (note that $T$ depends on $\epsilon$ rather than $n$ and $p$).

%%%%%

\subsection{Heavy-Tailed Distributions}
\label{sec: app heavy-tail}

Throughout this subsection, we work under the following assumptions:
\begin{assumption}
\label{AssHeavy}
Suppose the link function $\Phi$ satisfies inequalities~\eqref{eq: phi' assump} and~\eqref{eq: phi^(t) assump} for $k \in \{1,2\}$ and $t \in \{2, 4\}$, and inequality~\eqref{eq: phi^(t) assump bounded} for $t \in \{2,3\}$. Also suppose $\cR(\theta)$ is $m$-strongly convex, i.e., $mI \preceq \nabla^2 \cR(\theta)$ uniformly over $\theta$, and $L$ and $M$ are defined as in equation~\eqref{EqnLM}.
\end{assumption}

%\textcolor{red}{Should we also analyze the sample median here? Note that the proofs of Theorems~\ref{ThmPure} and~\ref{ThmDamped} (cf.\ inequalities~\eqref{EqnOptErr} and~\eqref{EqnChain}, respectively) show that $\|\theta_t - \theta^*\|_2$ remains bounded (where the bound depends on $\theta_0$ and the problem parameters). Then the error bounds on the median estimates will follow from Lemma~\ref{LemMedian}.}

We then have the following result, proved in \ref{AppThmGLMHeavy}:

\begin{theorem}
\label{ThmGLMHeavy}
Let  $\{z_i\}_{i=1}^n$ be $n$ i.i.d.\ samples drawn from a heavy-tailed distribution $P$ that satisfies the GLM in \eqref{eq: model GLM}, and suppose the conditions of Assumptions~\ref{AssCov} and~\ref{AssHeavy} are satisfied.
Define $(\eta, c_2, \zeta, \gamma_g, \gamma_h)$ as in Theorem~\ref{ThmPure} and $\gamma$ as in Theorem~\ref{ThmDamped}.
Let $\delta>0$.
Suppose $n$ satisfies
\begin{align}
\label{eq: ThmGLMHeavy n asmp}
     n > C \max\left\{ \frac{p}{\widehat{\alpha}_g^2}, \frac{p}{\widehat{\beta}_g^2}, \frac{p^2 \|\Sigma_x\|_2^2}{\widehat{\beta}_h^2}, \frac{1}{\zeta^2} \right\} \log\left(\frac{1.4}{\delta} \right),
\end{align}
where $(\widehat{\alpha}_g, \widehat{\beta}_g, \widehat{\beta}_h)$ are as defined in Lemma~\ref{lem: thm asmps}.
Then applying Algorithm~\ref{alg:robust Newton} on $\{z_i\}_{i=1}^n$, with initialization $\theta_0\in \Theta$ and number of iterations
 \begin{align*}
    T \ge  \frac{\cR(\theta_0) - \cR(\theta^*)}{\gamma/2} + \log_2\log_2\left(\frac{6c_2L}{m^2}\right),
\end{align*}
returns an output such that 
\begin{equation*}
\|\theta_T - \theta^*\|_2 \leq \frac{12c_2}{m} = O\left(\sqrt{\frac{p^2}{n}}\right),
\end{equation*}
with probability at least $1 - T\delta \left(5 + \Big\lceil\frac{\log\left(\frac{m}{M}\left(\frac{1}{2} - \kappa_1\right)\right)}{\log \kappa_2}\Big\rceil \right)$. 
\end{theorem}

\begin{remark}
\label{Rem4WiseHessian}
Again, assuming 4-wise independence of the coordinates of the covariate distribution, we can reduce the dimension-dependence of the bounds (cf.\ Remark~\ref{Rem4Wise}). We then take $\beta_h \asymp \sqrt{p} \|\Sigma_x\|_2^2$, to obtain an estimation error bound of the form $\|\thetahat - \theta^*\|_2 = O\left(\sqrt{\frac{p}{n}}\right)$.
\end{remark}

We also briefly present a separate line of analysis that allows us to improve the estimation error rates from $O\left(\sqrt{\frac{p^2}{n}}\right)$ to $O\left(\sqrt{\frac{p}{n}}\right)$ under milder assumptions than the distributional assumptions mentioned in Remark~\ref{Rem4WiseHessian}. Theorem 1 of Minsker~\cite{minsker2018sub} discusses an estimator $\hat{T}$ for the mean $\E[Y]$ of i.i.d.\ observations $Y_1, \dots, Y_n \in \mathbb{R}^{p \times p}$, with sub-Gaussian rates; i.e.,
\begin{equation*}
\mprob\left(\|\hat{T} - \E[Y]\|_2 \ge \|\E[Y^2]\|_2^{1/2} \sqrt{\frac{t}{n}}\right) \le 2p\exp\left(-\frac{t}{2}\right),
\end{equation*}
for any $t > 0$. In our setting, we have
\begin{equation*}
\E[Y^2] = \E\left[(\Phi''(x_i^T \theta))^2 x_i x_i^T x_i x_i^T\right] \le \overline{M}_{\Phi, 2}^2 \E\left[\|x_i\|_2^2 \cdot x_i x_i^T\right].
\end{equation*}
In particular, for any unit vector $v \in \real^p$, we have
\begin{equation*}
v^T \E[Y^2] v \le \overline{M}_{\Phi, 2}^2 \E\left[\|x_i\|_2^2 \cdot (x_i^T v)^2\right] \le \overline{M}_{\Phi, 2}^2 \sqrt{\E\left[\|x_i\|_2^4\right] \cdot \E\left[(x_i^T v)^4\right]}.
\end{equation*}
By assumption, we have
\begin{equation*}
\E\left[(x_i^T v)^4\right] \le \widetilde{C}_4 \left(\E\left[(x_i^T v)^2\right]\right)^2 \le \widetilde{C}_4 \cdot \lambda_{\max}(\Sigma_x)^2 = O(1).
\end{equation*}
Furthermore,
\begin{equation*}
\E\left[\|x_i\|_2^4\right] = \E\left[\left(\sum_{i=1}^p x_{ij}^2\right)^2\right],
\end{equation*}
and as argued in the proof of Lemma~\ref{lem: Hess est GLM}, this is $O(p^2)$. Altogether, we conclude that $\|\E[Y^2]\|_2 = O(p)$, so using this high-probability bound in place of Lemma~\ref{lem: Hessian est heavy-tail} leads to an improvement in Proposition~\ref{PropErrsHeavy} with $\beta_h = O\left(\sqrt{\frac{p}{n}}\right)$.

As mentioned at the end of Section~\ref{SecHT}, the computational complexity of Algorithm~\ref{alg:heavy tailed estimator} is $\widetilde{O}(n+p)$, leading to an overall runtime of $\widetilde{O}\left(T(n+p^2)\right)$ for the robust Newton method.

%%%%%

\section{Simulations}
\label{SecSims}

We note that in our simulations, we have implemented the code from Lai et al.~\cite{lai2016} for agnostic mean estimation. In particular, the outlier truncation step is slightly different from the one analyzed in Prasad et al.~\cite{prasad2018}, and consequently also in our theorems above.

\subsection{Huber's Contamination Model}

We begin with simulations for linear and logistic regression in Huber's contamination model.

\subsubsection{Linear Regression}
\label{SecSimLin}

For our simulations, we set the dimension to be $p=10$ and the number of data points to be $n=1000$. We simulated the clean covariates as $x_i \sim N(0, I_p)$, with corresponding responses $y_i = x_i^T \theta^* + w_i$, where $w_i \sim N(0, 0.1)$ is i.i.d.\ noise and the true parameter is $\theta^* = \frac{1}{\sqrt{p}}(1,1, \dots,1)$. We simulated the outlier covariates as $x_i \sim N(0, p^2 I_p)$, with corresponding responses $y_i = 0$.

Figure~\ref{fig:linreg_1} shows the results for Robust Newton's Method (RNM), Robust Gradient Descent (RGD), and ordinary least squares (OLS). We used the initialization $\theta_0 = (0.4, \dots, 0.4) + 10w$, with $w \sim \mathcal{N}(0,I_p)$, for both RNM and RGD. For RNM, we used the backtracking linesearch parameters $\kappa_1=0.01, \kappa_2=0.5$, and $\zeta=10^{-8}$. For RGD, we used stepsize $\eta=0.1$. We repeated the algorithm three times with contamination fractions $\epsilon=0.1,0.2$, and $0.3$. As seen in the figure, the statistical error indeed decreases quite quickly for RNM in comparison to RGD.

\begin{figure}[h]
\centering
\includegraphics[scale=0.5]{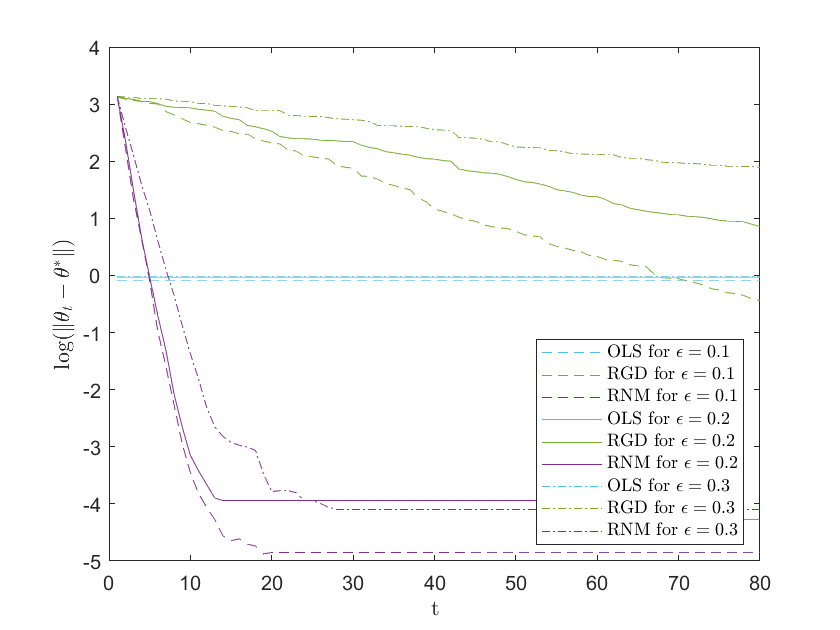}
\caption{Error log($\| \theta_t - \theta^* \|_2$) with respect to each iteration of Robust Newton's Method (RNM), Robust Gradient Descent (RGD), and ordinary least squares (OLS), for linear regression with Huber contamination.}
\label{fig:linreg_1}
\end{figure}

\subsubsection{Logistic Regression}\label{subsec: simul log reg}

Next, we generated data from a logistic model with $p=10$, $n=1000$, and $\theta^* = (1/\sqrt{p}, \ldots, 1/\sqrt{p})$, where we sampled the covariates as $x_i \sim \cN(0, \cI_p)$ and sampled $y_i \in \{0,1\}$ such that $p(y_i=1|x_i) = \frac{1}{1+e^{-x_i^T\theta^*}}$. We then randomly changed an $\epsilon$ fraction of the labels to be either $0$ or $1$, with equal probability. For various values of $\epsilon$, we ran Robust Gradient Descent (RGD) and Robust Newton's Method (RNM), and plotted the parameter error in Figure~\ref{fig:logreg_1}. For RNM, we used the same backtracking linesearch parameters as in the case of linear regression with Huber contamination. For RGD, we used a stepsize of $\eta = 3$. As seen in the figure, the statistical error again decreases more quickly for RNM than for RGD.
 
\begin{figure}[h]
\centering
\includegraphics[scale=0.25]{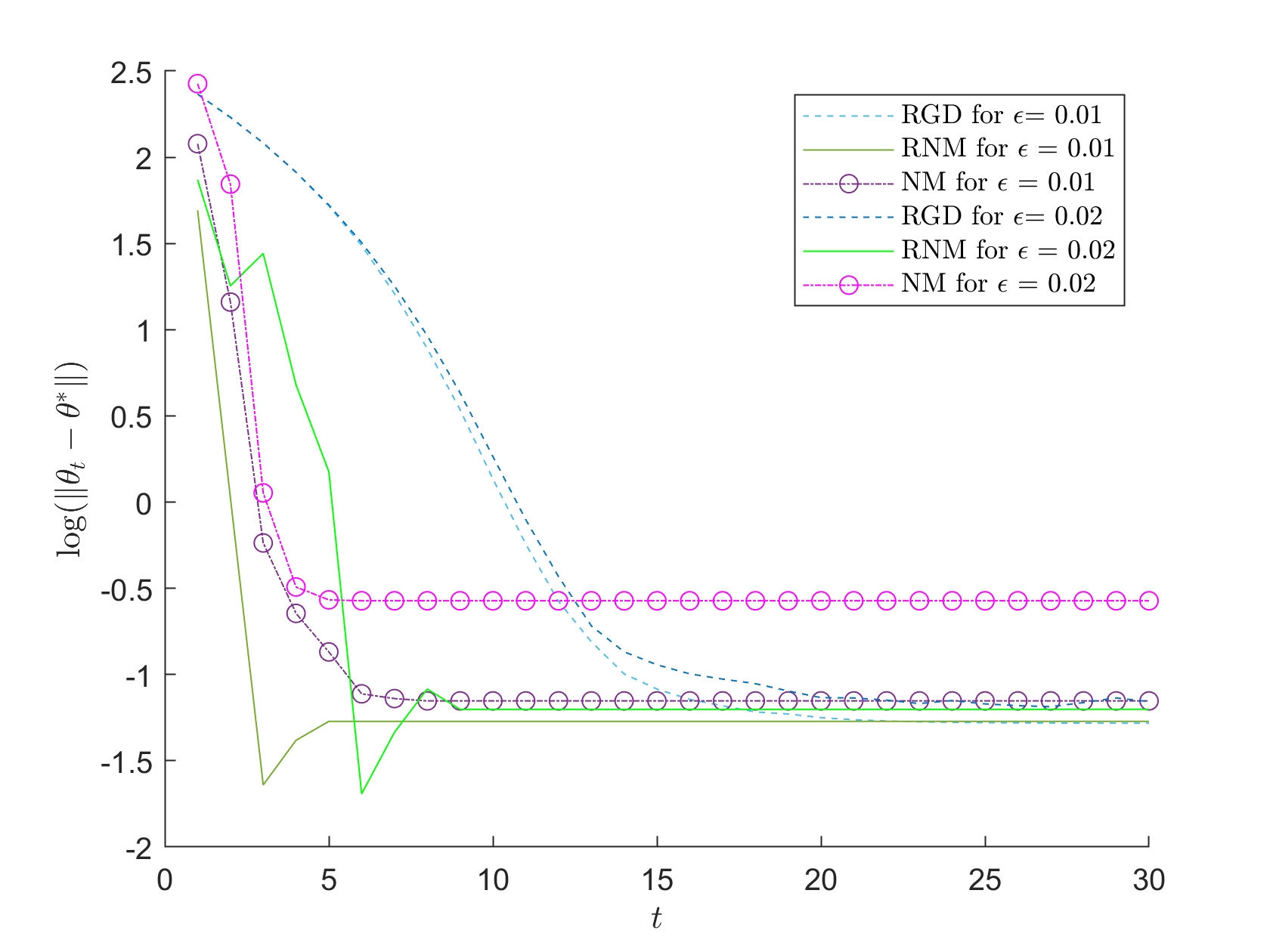}
\caption{Error log($\| \theta_t - \theta^* \|_2$) with respect to each iteration for Robust Newton's Method (RNM) and Robust Gradient Descent (RGD), for logistic regression with Huber contamination. The behavior of the non-robust optimizer, found using Newton's method (NM), is also shown for reference.}
\label{fig:logreg_1}
\end{figure}

\subsection{Heavy-Tailed Data}

For heavy-tailed data, we took $p=10$ and $n=1000$. We generated the covariates $x_i \sim N(0, I_p)$ and the corresponding responses $y_i=x_i^T \theta^* + w_i$, with $w_i$ following a Pareto distribution with variance $\sigma^2$ and tail-index parameter $\beta$. We set the regression parameter $\theta^* = \frac{1}{\sqrt{p}}(1,1, \dots,1)$. 

Figure~\ref{fig:linreg_3} compares the results of Robust Newton's Method (RNM), Robust Gradient Descent (RGD), and ordinary least squares (OLS). We used the initialization $\theta_0 = (10,10, \dots, 10)$ for both RNM and RGD. For RNM, we used the backtracking linesearch parameters $\kappa_1=0.01, \kappa_2=0.5$, and $\zeta=1000$. For RGD, we used stepsize $\eta=0.1$. We repeated the algorithm three times, for $\sigma=0.5,1$, and $1.5$, all with $\beta=1$. As seen in the figure, the statistical error again decreases more quickly for RNM than for RGD.
%on the other hand, in this simulation, the final error of RGD is lower than that of RNM.

\begin{figure}[ht]
\centering
\includegraphics[scale=0.25]{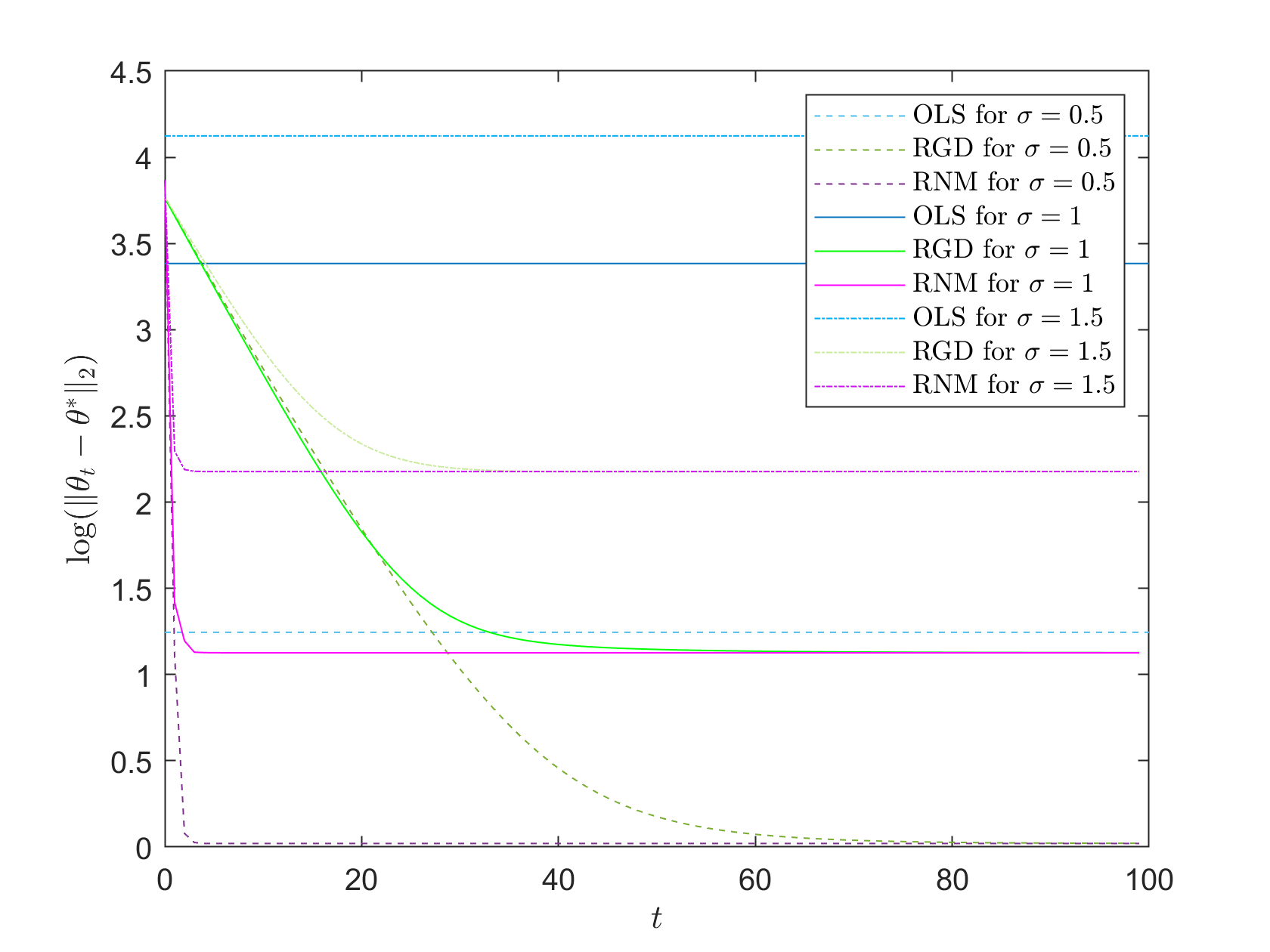}
\caption{Error log($\| \theta_t - \theta^* \|_2$) with respect to each iteration for Robust Newton's Method (RNM), Robust Gradient Descent (RGD), and ordinary least squares (OLS), for linear regression with heavy-tailed data.}
\label{fig:linreg_3}
\end{figure}

%%%%%

\section{Robust Hessian Estimation: The Conjugate Gradient Approach}
\label{SecCG}

%From the above discussion, we see that in order to apply Robust Newton's Method, we need to be able to approximate the Hessian correctly using information from a contaminated dataset. When we are not working in the linear regression setting we cannot use $\D^2 f(\theta) = \mathbb{E}[ x x^T]$ to estimate the Hessian because it is no longer true. Our $CovarianceEstimation$ algorithm cannot be generalised and applied estimate a general Hessian $\mathbb{E}[\D^2\mathcal{L}(\theta,z)]$ because the approximation error can get big.
In this section, we discuss an alternative to Newton's method (and present a robust variant thereof) which does not involve explicitly computing the Hessian. Inspired by Martens~\cite{james2010}, the idea is to estimate $\D^2 f(\theta) v$, for any vector $v$, using the approximation
\begin{equation}
\label{EqnHV}
h_v(\theta)=\frac{\D f(\theta + \delta v) - \D f(\theta)}{\delta},
\end{equation}
for some small $\delta > 0$. Note that in order to compute the Newton step $\dx$, we need to solve the system $\D^2 f(\theta) \Delta\theta = -\D f(\theta)$, which we will do using the conjugate gradient algorithm, which provides an iterative method for solving a linear system of the form $Ax = b$~\cite[Chapter 5]{wright1999numerical}. Our robust approach will involve using the robust gradient estimate $g(\theta)$ in place of $\D f(\theta)$.

%%%%%

%\subsection{Algorithm}

The details of the algorithm are provided in Algorithm~\ref{alg:CG Newton}. Note that we have specified that the CGNewtonStep subroutine for finding the Newton direction on each iteration of CGRobustNewton runs for $p$ steps, because in the noiseless case, the conjugate gradient method is known to terminate in at most $p$ steps.

%%%%%

\begin{algorithm}[H]
\centering
\caption{Conjugate Gradient Robust Newton's Method}\label{alg:CG Newton}
\begin{algorithmic}[1]
\Require Data samples $S = \{z_i\}_{i = 1}^n$, Number of iterations $T$, Initial guess $\theta_0\in \Theta$, Backtracking linesearch parameters $\kappa_1 \in (0, 0.5), \kappa_2 \in (0,1)$, and $\zeta$, Tolerance $\delta$
%\Require (Only if $\text{Type} = \text{Huber}$) Corruption Level $\epsilon$, Dimension $p$

\State 

\Function{CGRobustNewton($S, \xi, \theta_0, , \kappa_1, \kappa_2, \zeta$)}{}

\For{$t=0$ to $T-1$}

\State Compute losses  $ \{ \mathcal{L}(\theta_t,z_i) \}_{i=1}^n$  and gradients $ \{ \nabla \mathcal{L}(\theta_t,z_i) \}_{i=1}^n$
\State Compute Newton step $\Delta \theta_{nt} =  \text{{\sc CGNewtonStep}}(\theta_t)$
\State Compute stepsize $\alpha = \text{{\sc BacktrackingLineSearch}}(S, \theta_t, \Delta \theta_{nt}, g(\theta_t), \kappa_1, \kappa_2, \zeta)$
\State Update $\theta_{t+1} = \theta_t +  \alpha \Delta \theta_{nt}$

\EndFor

\Return $\theta_T$
\EndFunction

\State

\Function{CGNewtonStep($\theta$)}{}
\State Randomly initialize $\Delta\theta^{(0)} \in \Theta $
\State Compute gradient estimate $g(\theta) = \text{{\sc RobustGradientEstimate}}(S, \theta)$
\State Compute Hessian-vector product estimate $h_{\Delta\theta^{(0)}}(\theta) = \textsc {HVProduct}(\theta, \Delta\theta^{(0)})$ 
\State Set $r_0 = h_{\Delta\theta^{(0)}}(\theta) + g(\theta)$
\State Set $p_0 = -r_0$
\For{$k=1$ to $p-1$}
\State Compute Hessian-vector product estimate $h_{p_k}(\theta) = \textsc {HVProduct}(\theta, p_k)$ 
\State Set $\alpha_k = \frac{r_k^T r_k}{p_k^T h_{p_k}(\theta)}$
\State Set $\Delta\theta^{(k+1)}  = \Delta\theta^{(k)} + \alpha_k p_k$
\State Set $r_{k+1} = r_k + \alpha_k h_{p_k}(\theta)$
\State Set $\beta_{k+1} = \frac{r_{k+1}^T r_{k+1}}{r_k^T r_k}$
\State Set $p_{k+1} = -r_{k+1} + \beta_{k+1}p_k$
\EndFor

\Return $\Delta\theta^{(d)}$
\EndFunction

\State

\Function{HVProduct($\theta, v$)}{}
\State Compute gradient estimate $g(\theta) = \text{{\sc RobustGradientEstimate}}(S, \theta)$
\State Compute gradient estimate $g(\theta+\delta v) = \text{{\sc RobustGradientEstimate}}(S, \theta+\delta v)$

\Return $\frac{g(\theta+\delta v) - g(\theta)}{\delta}$

\EndFunction

\State

\Function{BacktrackingLineSearch($(S, \theta, \Delta \theta_{nt}), g(\theta), \kappa_1, \kappa_2, \zeta$)}{}

\State Set $\alpha = 1$
\While {$\text{{\sc RobustEstimate}} ( \{ \mathcal{L}(\theta +\alpha \Delta \theta_{nt}, z_i) \}_{i=1}^n) >  \text{{\sc RobustEstimate}} ( \{ \mathcal{L}(\theta, z_i) \}_{i=1}^n) + \kappa_1 \alpha g(\theta) \Delta\theta_{nt} + \zeta$}
\State Update $\alpha = \kappa_2 \alpha$
\EndWhile

\Return $\alpha$

\EndFunction

\end{algorithmic}
\end{algorithm}

\subsection{Convergence}

We sketch some ideas here; a rigorous proof giving rates of convergence of the robust conjugate gradient method is beyond the scope of this work. Focusing on the pure Newton phase, note that our analysis of the iterates of robust Newton's method essentially hinges on the Newton step $\dx$ satisfying the equation
\begin{equation}
\label{EqnNewtonStep}
\nabla f(\theta_t) = -\nabla^2 f(\theta_t) \dx + \chi_t,
\end{equation}
where the next iterate is then defined by $\theta_{t+1} = \theta_t + \dx$ and $\chi_t$ is a small, bounded error (cf.\ inequalities~\eqref{EqnExpansion} and~\eqref{EqnExpansion2}). In particular, we can bound $\chi_t$ using the fact that $\dx = -H(\theta_t)^{-1} g(\theta_t)$, and $\|g(\theta_t) - \nabla f(\theta_t)\|_2$ and $\|H(\theta_t) - \nabla^2 f(\theta_t)\|_2$ are small. In the case of the robust conjugate gradient method, we can again think of the conjugate gradient method as providing an approximate solution of the form
\begin{equation}
\label{EqnNewtonStep}
\nabla f(\theta_t) = -\nabla^2 f(\theta_t) \dxtil + \widetilde{\chi}_t,
\end{equation}
where successive iterates are then defined by $\thetatil_{t+1} = \thetatil_t + \dxtil$. Thus, the main challenge is to understand the propagation of errors when the conjugate gradient method is applied to solve the system $Ax = b$, but the matrix-vector pair $(A,b)$ is replaced by $(\widetilde{A}, \widetilde{b})$ on each iteration. To the best of our knowledge, this is actually an open question in optimization~\cite{greenbaum1989behavior, greenbaum1992predicting}. We note, however, that since our ultimate statistical estimation error bounds are all up to a small radius of, e.g., $O(\sqrt{\epsilon})$, we only need the output of the conjugate gradient method to be correct up to this error. In particular, as it is known that the exact conjugate gradient method terminates after $p$ steps~\cite{{wright1999numerical}}, it would for instance suffice to show that an inexact conjugate gradient method, where the error of $(\widetilde{A}, \widetilde{b})$ is also $O(\sqrt{\epsilon})$, only accumulates $O(\sqrt{\epsilon})$ error after $p$ steps.
Alternatively, one could try to derive a geometric rate of convergence (cf.\ equation (5.36) of Nocedal and Wright~\cite{wright1999numerical}), with an additional additive error term, for inexact conjugate gradient steps. Clearly, each iterate of the conjugate gradient method has computational complexity $O(p^2)$, since it involves a small handful of matrix/vector multiplications. Thus, the overall complexity of $p$ iterations would be $O(p^3)$, as well, leading to a computational complexity of $O(Tp^3)$ when combined with Newton's method. 

%In order to analyze the convergence behavior of iterates $\{\thetatil_t\}$ obtained from the robust conjugate gradient method, we can treat them as noisy versions of the iterates $\{\theta^t\}$ which would be obtained by running the usual (non-robust) Newton's method on the population-level objective. In particular, focusing on the pure Newton phase, each iteration of the robust algorithm involves making an update
%\begin{equation*}
%\thetatil_{t+1} = \thetatil_t + \dxtil,
%\end{equation*}
%where $\dxtil$ can be thought of as a noisy version of the Newton step $\dx$.

We also need to quantify the error terms introduced to conjugate gradient steps due to inexactness. This depends on the increment $\delta$ used in the finite-difference approximation of the Hessian term~\eqref{EqnHV}. Note that by a Taylor expansion, we have
\begin{equation*}
\D f(\theta + \delta v) = \D f(\theta) + \delta \D^2 f(\theta) v + C \delta^2,
\end{equation*}
for some constant $C$.
Thus, we have the error bounds
\begin{align*}
\|h_v(\theta) - \nabla^2 f(\theta) v\|_2 & = \left\|\frac{g(\theta + \delta v)-g(\theta)}{\delta} - \nabla^2 f(\theta) v\right\|_2 \\
& = \left\|\frac{g(\theta + \delta v)-g(\theta)}{\delta} - \frac{\nabla f(\theta + \delta v) - \nabla f(\theta)}{\delta} - C\delta\right\|_2 \\
& \le \frac{\|g(\theta + \delta v) - \nabla f(\theta + \delta v)\|_2}{\delta} + \frac{\|g(\theta) - \nabla f(\theta)\|_2}{\delta} + C\delta.
\end{align*}
If we had deviations bounds of the form~\eqref{eqn:Hess_estimator}, e.g., with $\alpha_g, \beta_g \asymp \sqrt{\epsilon}$, the optimal choice of $\delta$ would be $\delta \asymp \epsilon^{1/4}$.

In summary, we conjecture that the robust conjugate gradient method would allow us to incur an overall estimation error of $O(\epsilon^{1/4})$ in the case of Huber's $\epsilon$-contamination model, again at a quadratic convergence rate for the successive Newton iterates. Although this is a slower rate than the one derived in Section~\ref{SecApps} for GLMs, it may be applicable to a much wider range of settings. We also note that for the SEVER algorithm~\cite{diakoconf2019}, a rate of $O(\epsilon^{1/4})$ is derived for empirical risk minimization for a class of classification problems. If the above discussion could be made rigorous, it would also be extendable to the heavy-tailed setting in a straightforward manner.

\subsection{Simulations}

In Figure~\ref{fig:linreg_2}, we compare the Newton Conjugate Gradient Method (NCGM), Robust Gradient Descent (RGD), and ordinary least squares (OLS) on a linear model with Huber $\epsilon$-contaminated data, with the same setup as in Section~\ref{SecSimLin}. We used the initial parameter $\theta_0 = (1, \dots, 1) + 2w$, with $w \sim N(0, I_p)$, for both NCGM and RGD. For NCGM, we used the backtracking linesearch parameters $\kappa_1=0.01, \kappa_2=0.5$, and $\zeta=0.001$. We also used $\delta=10^{-9}$ for the estimation of Hessian-vector products. For RGD, we used stepsize $\eta=0.02$. We repeated the algorithm two times, with contamination fractions $\epsilon=0.01$ and $0.02$.

\begin{figure}[ht]
\centering
\includegraphics[scale=0.5]{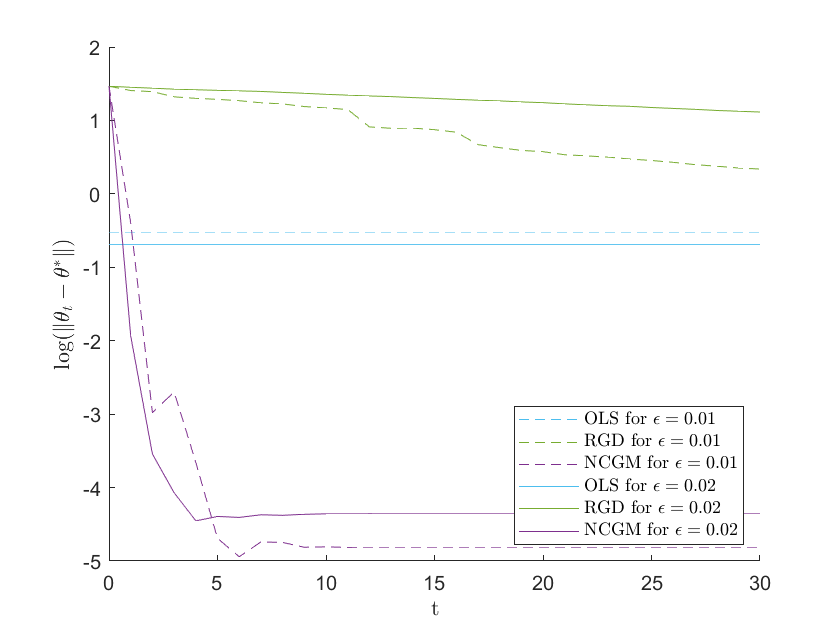}
\caption{Error log($\| \theta_n - \theta^* \|_2$) with respect to each iteration for the Newton Conjugate Gradient Method (NCGM), Robust Gradient Descent (RGD), and ordinary least squares (OLS), for linear regression with Huber contamination.}
\label{fig:linreg_2}
\end{figure}

In Figure~\ref{fig:linreg_4}, we compare the Newton Conjugate Gradient Method (NCGM), Robust Gradient Descent (RGD), and ordinary least squares (OLS) on a linear model with heavy-tailed data, again with the same setup as in Section~\ref{SecSimLin}. We used the initial parameter $\theta_0=(1.5, \dots, 1.5) + 2w$, with $w \sim N(0, I_p)$, for both NCGM and RGD. For NCGM, we used the backtracking linesearch parameters $\kappa_1=0.01, \kappa_2=0.5$, and $\zeta=0.00001$. We also used $\delta=10^{-10}$ for the estimation of Hessian-vector products. For RGD, we used stepsize $\eta=0.2$. We repeated the algorithm twice with $\sigma=0.5$ and $0.25$, with $\beta=0.7$.

\begin{figure}[ht]
\centering
\includegraphics[scale=0.5]{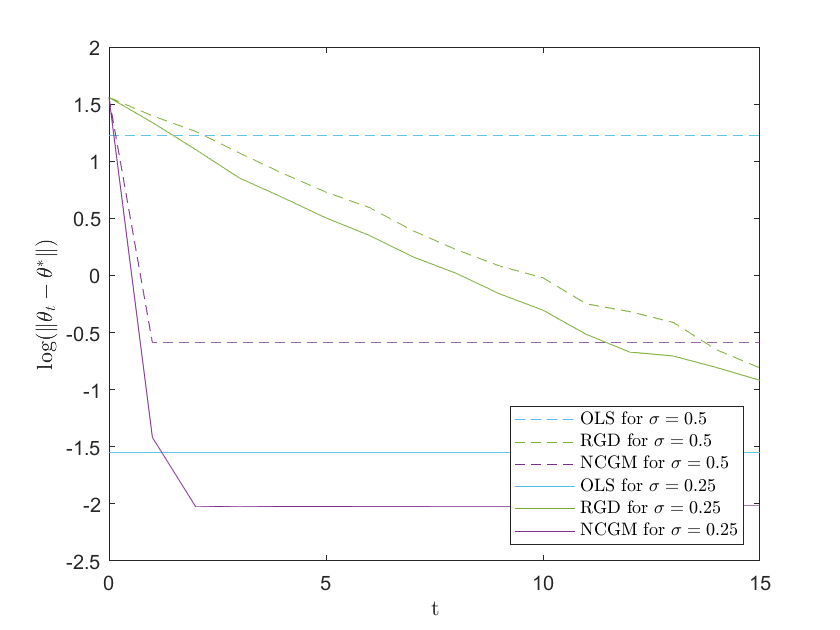}
\caption{Error log($\| \theta_n - \theta^* \|_2$) with respect to each iteration for the Newton Conjugate Gradient Method (NCGD), Robust Gradient Descent (RGD), and ordinary least squares (OLS), for linear regression with heavy-tailed data.}
\label{fig:linreg_4}
\end{figure}

In Figure~\ref{fig:logreg_2}, we compare the Newton Conjugate Gradient Method (NCGM) and Robust Gradient Descent (RGD) on a logistic model with Huber $\epsilon$-contamination. To generate the contaminated logistic data, we used the same procedure outlined in Section~\ref{subsec: simul log reg}. We also used the same hyperparameters for NCGM and RGD as in Section~\ref{subsec: simul log reg}.

\begin{figure}[ht]
\centering
\includegraphics[scale=0.25]{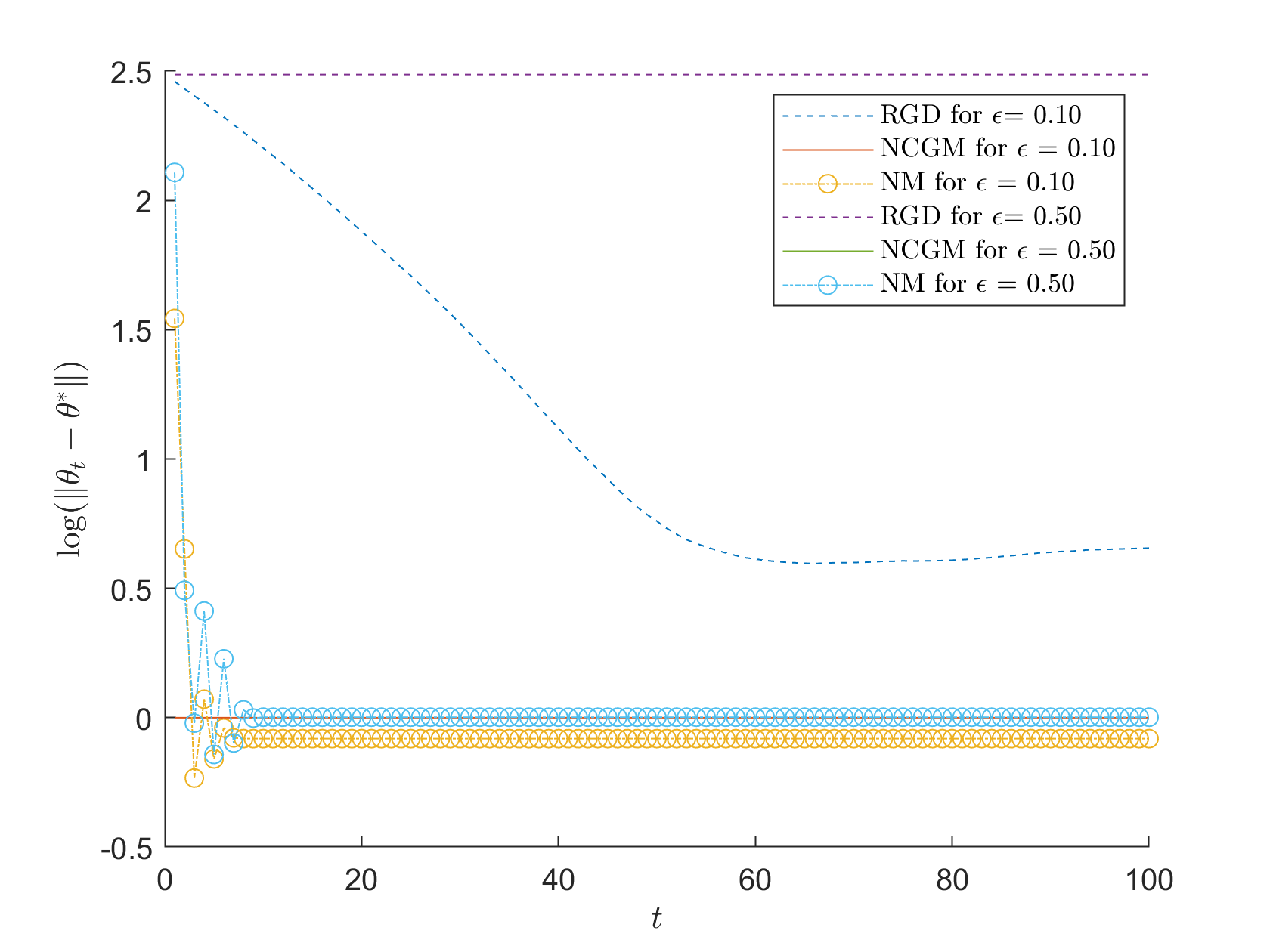}
\caption{Error log($\| \theta_n - \theta^* \|_2$) with respect to each iteration for the Newton Conjugate Gradient Method (NCGM) and Robust Gradient Descent (RGD), for logistic regression with Huber contamination. The behavior of the non-robust optimizer, found using Newton's method (NM), is also shown for reference.}
\label{fig:logreg_2}
\end{figure}

%\textcolor{red}{It would probably also be a good idea to provide a plot showing the accumulation of errors while running the conjugate gradient routine for obtaining an approximate Newton direction. Maybe we can do this by taking a particular iteration of the robust Newton method (i.e., the first step) and plot the errors $\|y_{k+1} - \dx\|_2$ as a function of $k$, where $\dx = \nabla^2 f(\theta_t)^{-1} \nabla f(\theta_t)$ is the ``true" Hessian). The hope is that the errors won't increase more than linearly.}

%%%%%

\section{Discussion}
\label{SecDiscussion}

We have presented a novel second-order method for robust parameter estimation, based on an adaptation of Newton's method where gradients and Hessians are computed in a robust manner on each iteration. In particular, we have shown that a variant of the backtracking linesearch algorithm will adaptively choose stepsizes in such a way that a finite number of iterates initially lie in a ``damped" phase of the algorithm, after which the algorithm enters a ``pure" phase where it only chooses stepsizes equal to 1 and converges quadratically to a small ball around the true parameter.

Under appropriate assumptions, our method shows clear computational advantages, both theoretically and empirically, in comparison to previously analyzed first-order methods. However, the general statements of Theorems~\ref{ThmGLMHuber} and~\ref{ThmGLMHeavy} leave much to be desired in terms of their dependence on $p$; with an infinite computational budget, the rates from robust gradient descent depend only logarithmically on $p$ in the Huber contamination model. It is thus natural to wonder whether improvements exist which either involve using less naive methods for computing robust Hessian matrices~\cite{minsker2018sub, minsker2022robust, cheng2019faster} or a robust variant of a quasi-Newton method, which does not even attempt to estimate the Hessian matrices as closely~\cite{wright1999numerical}. This opens up an interesting question of the ``right" type of Hessian estimate which interpolates between robust gradient descent and the robust Newton method presented here, achieving the smallest number of iterations necessary for a desired level of statistical accuracy.

Another plausible extension of our analysis that could be studied under a similar theoretical framework would be to use robust gradient and Hessian estimators which employ the estimation procedures of Diakonikolas et al.~\cite{diakonikolas2019recent} rather than those of Lai et al.~\cite{lai2016}; we note that this would allow us to also handle the setting of adversarially contaminated data, rather than i.i.d.\ data from either an $\epsilon$-contaminated or heavy-tailed model. We also note that the underlying assumption in our paper and all the aforementioned papers is that the clean data are drawn i.i.d.\ from a distribution. As pointed out during the review process, extensions to the heteroscedastic case would be quite fascinating, but are beyond the scope of our current work. Even in the case of univariate mean estimation with no contamination, the analysis quickly becomes quite complicated~\cite{pensia2022estimating}.

It would also be interesting and practically important to devise robust second-order algorithms appropriate for higher-dimensional data. For moderate to large $p$ (even in settings where $p < n$), implementing the robust version of Newton's method can become more tedious, since it involves robustly computing $p \times p$ matrices and then inverting them on each iteration. In the truly high-dimensional case ($p > n$), even the canonical version of Newton's method must be modified, since the Hessian matrix becomes rank-deficient. This raises the question of whether it would be beneficial to analyze a robust inexact second-order algorithm, instead, where the Hessian matrix need not be approximated as closely. In the truly high-dimensional setting, combining this with regularization would be a natural direction for future work.

Finally, we have proposed the robust conjugate gradient method as an alternative second-order algorithm which, though based on Newton's method, only requires computing robust gradients rather than needing to separately compute robust Hessians. This method could potentially enjoy the fast convergence benefits of Newton's method while bypassing some of the computational issues in higher dimensions. However, a rigorous analysis of the robust conjugate gradient method is beyond the current scope of this paper---in particular, it would involve carefully tracking the propagation of errors through iterates of the conjugate gradient method, which has remained a long-standing open problem. We note that any error bounds on successive conjugate gradient iterates could then easily be plugged into our proofs to obtain quadratic convergence to an appropriate ball around the true parameter.

%%%%%

\appendix

\section{Proofs of Optimization-Theoretic Results}
\label{SecOptProofs}

We now provide the proofs of the convergence results stated in Section~\ref{SecGeneral}.

\subsection{Proof of Theorem~\ref{ThmPure}}
\label{AppThmPure}

Our first step is to show that backtracking linesearch chooses unit steps whenever the gradient is small, i.e., $\|\nabla f(\theta_t)\|_2 < \eta$. In other words, we want to prove that
\begin{equation*}
\tilde{f}(\theta+\dx) \leq \tilde{f}(\theta) - \kappa_1 \tilde{\lambda}(\theta)^2 + \zeta,
\end{equation*}
where $\theta = \theta_t$ denotes the iterate, $\tilde{f}$ denotes the robust estimate of $f$, and we have defined the noisy Newton decrement
\begin{equation}
\label{EqnNewtonDec}
\tilde{\lambda}(\theta) := \left(g(\theta)^T H^{-1}(\theta) g(\theta)\right)^{1/2}.
\end{equation}
Recall that $\dx = -H(\theta_t)^{-1} g(\theta_t)$.
Note that since $\|\nabla f(\theta_t)\|_2 < \eta$, we have~\cite[Equation (9.11)]{boyd2004}
\begin{equation}
\label{EqnOptErr}
\|\theta_t - \theta^*\|_2 \le \frac{2}{m} \|\nabla f(\theta_t)\|_2 < \frac{2\eta}{m} := \gamma_0.
\end{equation}
In particular, this implies a bound of $\gamma_g := \alpha_g \gamma_0 + \beta_g$ on the error of the gradient, and a bound of $\gamma_h := \alpha_h \gamma_0 + \beta_h$ on the error of the Hessian, according to Definitions~\ref{defn: robust gradient estimator} and~\ref{defn: robust Hessian estimator}.
We will show that
\begin{equation}
\label{EqnExit}
f(\theta+\dx) \leq f(\theta) - \kappa_1 \tilde{\lambda}(\theta)^2 + \frac{\zeta}{2},
\end{equation}
from which the desired result clearly follows by the accuracy bound~\eqref{EqnMedianErr} on the robust estimates and the triangle inequality.

Note that $\tilde{\lambda}(\theta)^2 = \Delta \theta_{nt}^T H(\theta) \Delta \theta_{nt}$, implying that
\begin{equation}
\label{EqnLambdaLower}
\tilde{\lambda}(\theta)^2 \geq (m-\gamma_h) \| \Delta\theta_{nt} \|_2^2 > \frac{m}{2} \|\Delta\theta_{nt} \|_2^2
\end{equation}
(where we assume $\gamma_h \le \frac{m}{2}$). Furthermore, by the Lipschitz condition, for $u \ge 0$, we have
\begin{equation*}
\|\nabla^2 f(\theta + u \dx) - \nabla^2 f(\theta)\|_2 \le uL\|\dx\|_2,
\end{equation*}
so
\begin{equation}
\label{EqnOrange}
\left|\dx^T \left(\nabla^2 f(\theta + u \dx) - \nabla^2 f(\theta)\right) \dx\right| \le u L\|\dx\|_2^3.
\end{equation}
Defining $\fbar(u) := f(\theta + u\dx)$, we have $\fbar''(u) = \dx^T \nabla^2 f(\theta + u\dx) \dx$, so we can rewrite inequality~\eqref{EqnOrange} as
\begin{equation*}
|\fbar''(u) - \fbar''(0)| \le uL\|\dx\|_2^3,
\end{equation*}
implying that
\begin{equation*}
\fbar''(u) \le \fbar''(0) + uL\|\dx\|_2^3 \le \fbar''(0) + uL\left(\frac{2}{m}\right)^{3/2} \tilde{\lambda}(\theta)^3,
\end{equation*}
using inequality~\eqref{EqnLambdaLower}. Integrating with respect to $u$ gives
\begin{equation*}
\fbar'(u) \le \fbar'(0) + u\fbar''(0) + \frac{u^2L}{2}\left(\frac{2}{m}\right)^{3/2} \tilde{\lambda}(\theta)^3,
\end{equation*}
and a second integration gives
\begin{equation}
\label{EqnThird}
\fbar(u) \le \fbar(0) + u\fbar'(0) + \frac{u^2}{2}\fbar''(0) + \frac{u^3L}{6}\left(\frac{2}{m}\right)^{3/2} \tilde{\lambda}(\theta)^3.
\end{equation}

Now note that
\begin{align}
\label{EqnFirst}
\fbar'(0) & = \nabla f(\theta)^T \dx \notag \\
& = -\nabla f(\theta)^T H^{-1}(\theta) g(\theta) \notag \\
& = -\tilde{\lambda}(\theta)^2 + (g(\theta) - \nabla f(\theta))H^{-1}(\theta) g(\theta) \notag \\
& \le -\tilde{\lambda}(\theta)^2 + \gamma_g \frac{1}{m - \gamma_h} \left(\|\nabla f(\theta)\|_2 + \gamma_g\right) \notag \\
& \le -\tilde{\lambda}(\theta)^2 + \gamma_g \frac{1}{m - \gamma_h} \left(\eta + \gamma_g\right),
\end{align}
whereas
\begin{align}
\label{EqnSecond}
\fbar''(0) & = \dx^T \nabla^2 f(\theta) \dx \notag \\
& = \tilde{\lambda}(\theta)^2 + \dx^T \left(\nabla^2 f(\theta) - H(\theta)\right) \dx \notag \\
& \le \tilde{\lambda}(\theta)^2 + \gamma_h \|\dx\|_2^2 \notag \\
& \le \tilde{\lambda}(\theta)^2\left(1 + \frac{2\gamma_h}{m}\right),
\end{align}
using the bound~\eqref{EqnLambdaLower} in the last inequality. Plugging inequalities~\eqref{EqnFirst} and~\eqref{EqnSecond} into inequality~\eqref{EqnThird} (with $u = 1$) then gives
\begin{equation*}
f(\theta + \dx) \le f(\theta) + \left(-\tilde{\lambda}(\theta)^2 + \frac{\gamma_g(\eta + \gamma_g)}{m - \gamma_h}\right) + \frac{\tilde{\lambda}(\theta)^2}{2} \left(1 + \frac{2\gamma_h}{m}\right) + \frac{L}{6}\left(\frac{2}{m}\right)^{3/2} \tilde{\lambda}(\theta)^3.
\end{equation*}
Finally, note that
\begin{equation*}
\tilde{\lambda}(\theta) \le \|H^{-1/2}(\theta)\|_2 \cdot \|g(\theta)\|_2 \le \frac{\|\nabla f(\theta)\|_2 + \gamma_g}{\sqrt{m - \gamma_h}} \le \frac{\eta + \gamma_g}{\sqrt{m - \gamma_h}}.
\end{equation*}
Since
\begin{equation*}
\zeta \ge 2\left(\frac{\gamma_g \cdot 2\eta}{m/2} + \frac{\gamma_h}{m} \left(\frac{2\eta}{\sqrt{m/2}}\right)^2\right)
\end{equation*}
and using the assumptions $\gamma_h \le \frac{m}{2}$ and $\gamma_g \le \eta$, we then have
\begin{equation*}
f(\theta + \dx) \le f(\theta) - \tilde{\lambda}(\theta)^2\left(\frac{1}{2} - \frac{L}{6} \left(\frac{2}{m}\right)^{3/2} \tilde{\lambda}(\theta)\right) + \frac{\zeta}{2}.
\end{equation*}
In particular, if $\tilde{\lambda}(\theta) \le \frac{3-6\kappa_1}{L(2/m)^{3/2}}$, which is guaranteed if $\eta$ is chosen sufficiently small so that
\begin{equation*}
\frac{\eta + \gamma_g}{\sqrt{m - \gamma_h}} \le \frac{3-6\kappa_1}{L(2/m)^{3/2}},
\end{equation*}
then inequality~\eqref{EqnExit} is indeed satisfied. We can guarantee this last inequality by taking $\eta \le \frac{3m^2(1-2\kappa_1)}{8L}$, assuming $\gamma_h \le \frac{m}{2}$ and $\gamma_g \le \eta$.

To derive the geometric convergence rate~\eqref{EqnContract}, we will use induction. We first establish an inequality of the form
\begin{equation}
\label{EqnRecursion}
\| \D f(\theta + \dx) \|_2 \le c_1 \| \D f(\theta) \|_2^2 +c_2,
\end{equation}
assuming $\|\nabla f(\theta)\|_2 < \eta$. Note that
\begin{align}
\label{EqnExpansion}
\| \D f(\theta+ \dx) \|_2 & = \| \D f(\theta+ \dx) - g(\theta) - H(\theta) \dx \|_2 \notag \\
& \le\| \D f(\theta+ \dx) - \nabla f(\theta) - \nabla^2 f(\theta) \dx \|_2 + \gamma_g + \gamma_h \|\dx\|_2 \notag \\
& = \left\|\int_0^1 \left(\nabla^2 f(\theta + u\dx) - \nabla^2 f(\theta)\right) \dx du\right\|_2 + \gamma_g + \gamma_h \|\dx\|_2 \notag \\
& \le \frac{L}{2} \|\dx\|_2^2 + \gamma_g + \gamma_h \|\dx\|_2,
\end{align}
using the Lipschitz condition in the second inequality. Next, we use the bound
\begin{equation*}
\|\dx\|_2 = \|H^{-1}(\theta) g(\theta)\|_2 \le \frac{1}{m - \gamma_h} \left(\|\nabla f(\theta)\|_2 + \gamma_g\right) \le \frac{2}{m} \left(\|\nabla f(\theta)\|_2 + \gamma_g\right),
\end{equation*}
assuming $\gamma_h \le \frac{m}{2}$. Plugging back into inequality~\eqref{EqnExpansion} gives
\begin{align}
\label{EqnExpansion2}
\|\nabla f(\theta + \dx)\|_2 & \le \frac{L}{2} \left(\frac{2\left(\|\nabla f(\theta)\|_2 + \gamma_g\right)}{m}\right)^2 + \gamma_g + \gamma_h \left(\frac{2\left(\|\nabla f(\theta)\|_2 + \gamma_g\right)}{m}\right) \notag \\
& = \frac{2L}{m^2} \|\nabla f(\theta)\|_2^2 + \|\nabla f(\theta)\|_2 \left(\frac{4\gamma_g L}{m^2} + \frac{2\gamma_h}{m}\right) + \left(\frac{2L\gamma_g^2}{m^2} + \gamma_g + \frac{2\gamma_g \gamma_h}{m}\right) \notag \\
& \le \frac{2L}{m^2} \|\nabla f(\theta)\|_2^2 + \eta \left(\frac{4\gamma_g L}{m^2} + \frac{2\gamma_h}{m}\right) + \frac{2L\gamma_g^2}{m^2} + \gamma_g + \frac{2\gamma_g \gamma_h}{m},
\end{align}
giving inequality~\eqref{EqnRecursion} with $c_1 = \frac{2L}{m^2}$ and $c_2 = \eta \left(\frac{4\gamma_g L}{m^2} + \frac{2\gamma_h}{m}\right) + \frac{2L\gamma_g^2}{m^2} + \gamma_g + \frac{2\gamma_g \gamma_h}{m}$. In particular, $c_2$ can be made small if we choose $\gamma_g$ and $\gamma_h$ small enough, and we will assume that $c_2 \le \frac{\eta}{2}$. We will also assume that $c_1 c_2 \le \frac{1}{12}$.

We are now ready for our induction. Using the notation $y_t := c_1 \| \D f(\theta_t) \|_2$, we will prove that $y_t < c_1\eta$ and $y_t \le y_0^{2^t} + c_1 c_2$ for all $t \ge 1$. For the base case $t=1$, note that
\begin{equation*}
y_1 \le y_0^2 + c_1 c_2 = y_0^{2^1} + c_1 c_2,
\end{equation*}
using inequality~\eqref{EqnRecursion}. Furthermore, since $y_0 < c_1\eta < \frac{1}{2}$ and $c_2 \le \frac{\eta}{2}$ by assumption, we have
\begin{equation*}
y_1 \le \frac{y_0}{2} + \frac{c_1\eta}{2} < \frac{c_1\eta}{2} + \frac{c_1 \eta}{2} = c_1 \eta.
\end{equation*}
For the inductive step, suppose $t \ge 1$, and we have $y_s < c_1 \eta$ and $y_s \le y_0^{2^s} + 3c_1 c_2$ for all $s \le t$. Then by inequality~\eqref{EqnRecursion}, we have
%if $y_s < c_1\eta$ for all $0 \le s \le t$, then $y_{t+1} < c_1 \eta$ and
%\begin{equation*}
%y_{t+1} \le y_0^{2^{t+1}} + 3c_1c_2.
%\end{equation*}
%For the base case $t=0$, note that $y_0 < c_1\eta$ by assumption. For the inductive step, suppose $y_s < c_1\eta$ for all $0 \le s \le t$, and note that by the above calculations, we have
\begin{equation}
\label{EqnQuadratic}
y_{t+1} \le y_t^2 + c_1c_2.
\end{equation}
Furthermore, using the assumption $\eta \le \frac{m^2}{4L}$, we have $y_t < \frac{1}{2}$, so if $c_2 \le \frac{\eta}{2}$, this implies that
\begin{equation*}
y_{t+1} \le \frac{y_t}{2} + \frac{c_1 \eta}{2} < \frac{c_1 \eta}{2} + \frac{c_1 \eta}{2} \le c_1\eta.
\end{equation*}
By inequality~\eqref{EqnQuadratic} and the induction hypothesis, we now write
\begin{align*}
y_{t+1} & \le y_t^2 + c_1 c_2 \\
& \le \left(y_0^{2^t} + 3c_1 c_2\right)^2 + c_1 c_2 \\
& = y_0^{2^{t+1}} + 6c_1 c_2 y_0^{2^t} + 9c_1^2 c_2^2 + c_1 c_2 \\
& \le y_0^{2^{t+1}} + \frac{3}{2} c_1 c_2 + \frac{3}{4} c_1 c_2 + c_1 c_2 \\
& \le y_0^{2^{t+1}} + 3c_1 c_2,
\end{align*}
using the assumption $c_1 c_2 \le \frac{1}{12}$. This completes the induction.

Thus,
\begin{equation*}
c_1 \|\nabla f(\theta_{t})\|_2 \le \left(\frac{1}{2}\right)^{2^{t}} + 3c_1 c_2.
\end{equation*}
%\begin{equation*}
%c_1 \left(\| \D f(\theta_t) \|_2 +\frac{c_2}{2 c_1}\right) \leq \left(\frac{1}{2}\right)^{2^t}+ 2 d_1,
%\end{equation*}
%so
%\begin{equation*}
%\frac{4L}{2 m^2} \| \D f(\theta_t) \|_2 \leq (\frac{1}{2})^{2^t}+ d_2 \se,
%\end{equation*}
%where $d_2  = 2 d_1  -\frac{c_2}{2}= \frac{c_2}{2}-\frac{c_2^2}{2}+2 c_3 c_1$. Therefore, we have %proved that $\| \D f(\theta_t) \|_2 \leq \mathcal{O}(e+E)$ as $t \rightarrow \infty$.
Applying inequality~\eqref{EqnOptErr} then gives the convergence rate
\begin{equation*}
\|\theta_{t+1} - \theta^*\|_2 \le \frac{2}{m} \cdot \frac{m^2}{2L} \left(\left(\frac{1}{2}\right)^{2^{t+1}} + \frac{6L}{m^2} c_2\right),
\end{equation*}
completing the proof.

%%%%%

\subsection{Proof of Theorem~\ref{ThmDamped}}
\label{AppThmDamped}

First, we show that we have an upper bound $\gamma_0' := \frac{2}{m} \sqrt{2M\left(f(\theta_0) - f(\theta^*)\right)}$ on $\|\theta_t - \theta^*\|_2$. We can then translate this into upper bounds $\gamma_g' := \alpha_g \gamma_0' + \beta_g$ and $\gamma_h' := \alpha_h \gamma_0' + \beta_h$ on the gradient and Hessian deviations, respectively.
%Furthermore, since $(\alpha_g, \beta_g, \alpha_h, \beta_h)$ may be made arbitrarily small by assumption, we can make the deviation bounds $\gamma_g := \alpha_g \gamma_0 + \beta_g$ and $\gamma_h := \alpha_h \gamma_0 + \beta_h$ on the gradient and Hessian arbitrarily small, as well.

By the result of Theorem~\ref{ThmPure}, we must have $\|\nabla f(\theta_s)\|_2 \ge \eta$ for all $0 \le s \le t$. Indeed, suppose $\|\nabla f(\theta_s)\|_2 < \eta$ for some $s < t$. Then by Theorem 1 (with the iterate $\theta_s$ relabeled as $\theta_0$), all successive iterates $\theta_{s+1}, \theta_{s+2}, \dots$, including $\theta_t$, would also need to have the norm of the gradient bounded by $\eta$, which contradicts the assumption that $\|\nabla f(\theta_t)\|_2 \ge \eta$.
We now show by induction that:
\begin{enumerate}
\item $f(\theta_s) \le f(\theta_0)$, and
\item $\|\theta_s - \theta^*\|_2 \le \gamma_0'$,
\end{enumerate}
for all $0 \le s \le t$. For the base case $s = 0$, note that claim (1) is obvious. We can establish claim (2) by noting that
\begin{equation}
\label{EqnChain}
\|\theta_s - \theta^*\|_2 \le \frac{2}{m} \|\nabla f(\theta_s)\|_2 \le \frac{2}{m} \sqrt{2M(f(\theta_s) - f(\theta^*))} \le \frac{2}{m} \sqrt{2M(f(\theta_0) - f(\theta^*))} = \gamma_0',
%\|\theta_s - \theta^*\|_2 \le \frac{2}{m} \left(f(\theta_s) - f(\theta^*)\right) \le \frac{2}{m} \left(f(\theta_0) - f(\theta^*)\right) = \gamma_0',
\end{equation}
using inequality~\eqref{EqnOptErr}, inequality (9.14) of Boyd and Vandenberghe~\cite{boyd2004}, and claim (1).

Turning to the inductive step, suppose claims (1) and (2) hold for all $s \le s'$, where $0 \le s' < t$. We wish to establish the claims for $s = s' + 1$. Note that if we prove claim (1), then claim (2) follows by the same chain of inequalities~\eqref{EqnChain}. Thus, it remains to establish claim (1).

Assuming $\gamma_g' \le \frac{\eta}{2}$, we have $\|g(\theta_s)\|_2 \geq \frac{\eta}{2}$ for all $0 \le s \le s'$ by claim (2), the fact that $\|\nabla f(\theta_s)\|_2 \ge \eta$, and the triangle inequality.
%We want to prove that $f(\theta + t \tilde{\Delta}\theta) - f(\theta) < -\gamma$ for some constant $\gamma$ and some appropriate $t \in (0,1)$. In other words we need to show that we have sufficient decrease of $f$, and hence sufficient convergence of the algorithm in the Damped Newton phase.
Using the same notation for the Newton decrement~\eqref{EqnNewtonDec}, we note that
\begin{equation}
\label{EqnNewtonDec2}
\tilde{\lambda}(\theta_s)^2 \geq \frac{1}{\sqrt{M+\gamma_h'}} \| g(\theta_s) \|_2^2 \geq \frac{ \eta^2 }{4\sqrt{2M}},
\end{equation}
where we assume $\gamma_h' \le M$.
%and
%\begin{equation*}
%g(\theta)^T \tilde{\Delta}\theta = -g(\theta)^T H^{-1}(\theta) g(\theta) = -\tilde{\lambda}(\theta)^2.
%\end{equation*}

First, we will prove that the exit condition of the $BacktrackingLineSearch$ function will be satisfied, i.e., we want to prove that
\begin{equation}
\label{EqnBackTrack}
\tilde{f}(\theta_{s'}+\alpha\dx) \leq \tilde{f}(\theta) - \kappa_1 \alpha \tilde{\lambda}(\theta_{s'})^2 + \zeta
\end{equation}
holds for small enough $\alpha$, where $\tilde{f}$ is the robust estimate of $f$. For convenience, we use the notation $\theta = \theta_{s'}$ in what follows. In fact, we will show that
\begin{equation}
\label{EqnBackTrack2}
f(\theta + \alpha \dx) \leq f(\theta) - \kappa_1 \alpha \tilde{\lambda}(\theta)^2
\end{equation}
for small enough $\alpha$, which clearly then implies inequality~\eqref{EqnBackTrack} by the triangle inequality and the condition~\eqref{EqnMedianErr}.

Consider the following:
\begin{align*}
f(\theta + \alpha \dx) & \leq f(\theta) + \alpha \nabla f(\theta)^T \dx + \frac{M}{2} \| \dx\|_2^2 \alpha^2 \\
& \le f(\theta) + \alpha g(\theta) \dx + \alpha \gamma_g' \|\dx\|_2 + \frac{M}{2} \tilde{\lambda}(\theta)^2 \frac{2}{m} \alpha^2 \\
& = f(\theta) - \alpha \tilde{\lambda}(\theta)^2 + \alpha \gamma_g' \|\dx\|_2 + \frac{M}{m} \tilde{\lambda}(\theta)^2 \alpha^2 \\
& \le f(\theta) - \alpha \tilde{\lambda}(\theta)^2 + \alpha \gamma_g' \tilde{\lambda}(\theta) \sqrt{\frac{2}{m}} + \frac{M}{m} \tilde{\lambda}(\theta)^2 \alpha^2,
%& \leq f(\theta) + t(g(\theta)-e_\theta)^T \dx + \frac{M}{2} \tilde{\lambda}(\theta)^2 \frac{2}{m}t^2 \\
%
%& \leq f(\theta) + tg(\theta)^T \dx + \frac{M}{m} \tilde{\lambda}(\theta)^2 t^2 - t e_\theta^T \dx \\
%
%& \leq f(\theta) - t \tilde{\lambda}(\theta)^2 + 
%\frac{M}{m} \tilde{\lambda}(\theta)^2 t^2 - t e_\theta^T \dx.
\end{align*}
where we use the relation $-\tilde{\lambda}(\theta)^2 = g(\theta)^T \Delta \theta_{nt}$ and inequality~\eqref{EqnLambdaLower}. Assuming $\gamma_g' \sqrt{\frac{2}{m}} \le \frac{1}{2} \sqrt{\frac{\eta^2}{4\sqrt{2M}}}$ and using inequality~\eqref{EqnNewtonDec2}, the last expression is upper-bounded as
\begin{align}
\label{EqnLS}
f(\theta + \alpha \dx) & \le f(\theta) - \frac{\alpha}{2} \tilde{\lambda}(\theta)^2 + \frac{M}{m} \tilde{\lambda}(\theta)^2 \alpha^2 \notag \\
& =  f(\theta) - \tilde{\lambda}(\theta)^2 \alpha \left(\frac{1}{2} - \frac{M}{m} \alpha \right).
\end{align}
%We have $ -t e_\theta^T  \dx \leq t \| e_\theta\|_2 \| \dx \|_2 \leq t c_1 e  \sqrt{ \frac{2}{m} } \tilde{\lambda}(\theta) $ .
%\\Pick $e$ such that $e \leq \frac{\eta^2 }{M c_1} \sqrt{\frac{m}{2}}$ (i.e., we want $ c_1 e  \sqrt{ \frac{2}{m} } \tilde{\lambda}(\theta) < \frac{1}{2} \tilde{\lambda}(\theta)^2 $). Then
%\begin{align*}
%f(\theta + t \tilde{\Delta}\theta) & \leq f(\theta) - t \tilde{\lambda}(\theta)^2 + 
%\frac{M}{m} \tilde{\lambda}(\theta)^2 t^2 +\frac{t}{2}\tilde{\lambda}(\theta)^2 \\
%
%& \leq f(\theta) - \frac{t}{2} \tilde{\lambda}(\theta)^2 + 
%\frac{M}{m} \tilde{\lambda}(\theta)^2 t^2 \\
%
%& \leq f(\theta) - \tilde{\lambda}(\theta)^2 t  
%\left(\frac{1}{2} - \frac{M}{m}t\right).
%\end{align*}
%Let $\alpha = \frac{1}{2} - \frac{M}{m}t$. We want $\alpha \geq 0$; hence, we need $t \leq \frac{m}{2M}$.
%\\
Hence, the condition~\eqref{EqnBackTrack2} is indeed satisfied for sufficiently small $\alpha$, i.e., $\alpha \le \frac{m}{M}\left(\frac{1}{2} - \kappa_1\right)$, and in particular, the linesearch procedure must return a stepsize satisfying $\alpha \ge \kappa_2 \frac{m}{M}\left(\frac{1}{2} - \kappa_1\right)$. Plugging such a stepsize into inequality~\eqref{EqnBackTrack}, we have
\begin{equation}
\label{EqnLineSearch}
f(\theta_{s'+1}) \leq f(\theta_{s'})  - \kappa_1 \cdot \kappa_2 \frac{m}{M} \left(\frac{1}{2} - \kappa_1\right) \cdot \frac{\eta^2}{4\sqrt{2M}} + \zeta \le f(\theta_0) - \gamma + \zeta \le f(\theta_0) - \frac{\gamma}{2},
\end{equation}
using inequality~\eqref{EqnNewtonDec2} and the induction hypothesis. This implies that claim (1) is true, completing the induction.

Finally, note that the inequality $f(\theta_{t+1}) - f(\theta_t) < -\frac{\gamma}{2}$ follows by the same argument in inequality~\eqref{EqnLineSearch} with $\theta = \theta_t$, completing the proof.
% := \kappa_1 \cdot \kappa_2 \frac{m}{M} \left(\frac{1}{2} - \kappa_1\right) \cdot \frac{\eta^2}{4\sqrt{2M}}$. This completes the proof.
%so
%\begin{equation*}
%f(\theta + t \tilde{\Delta}\theta) - f(\theta) \leq - \alpha t \frac{\eta^2}{2M}.
%\end{equation*}
%Let $t \geq \frac{\beta m}{M}$ for some $\beta \in (0,1)$. Then
%\begin{equation*}
%f(\theta + t \tilde{\Delta}\theta) - f(\theta) \leq - \alpha \beta \eta^2 \frac{m}{2M^2}.
%\end{equation*}
%Moreover, we would have the following restrictions on $e$ and $E$:
%\\$c_1 \| \D f(\theta_0) \|_2 +\frac{c_2}{2} \leq \frac{1}{2}$ becomes 
%$\frac{4L}{m^2}e+\frac{2}{m} \leq \frac{1}{2}$.
%\\For $d_1<\frac{1}{8}$ to hold, it suffices to have $c_2 < \frac{1}{10}$ and $c_3 < \frac{m^2}{256L}$, i.e., $\frac{4L}{m^2}e+\frac{2}{m}E< \frac{1}{10}$ and $e+\frac{4L}{2m^2}e+\frac{2}{m}eE< \frac{m^2}{256L}$.
%\\For $d_2 >0$, it suffices to have $c_2<1$ or $\frac{4L}{m^2}e+\frac{2}{m}E< 1$.
%\\Lastly, we want $E < \frac{m}{2}$.

\section{Proofs of Auxiliary Results for GLMs}

In this appendix, we prove some auxiliary results appearing in Section~\ref{SecApps}.

%%%%%

\subsection{Proof of Lemma~\ref{lem: Hess est GLM}}
\label{AppLemHessGLM}

From the definition of the loss function~\eqref{eq:loss GLM}, we have $\E[\D^2 \cL(\theta,z)] = \E[\Phi''(x_i^T\theta)x_ix_i^T]$.
By our assumptions on the boundedness of $\Phi''$ and bounded eighth moments of $x_i$, we see that the distribution of the flattened Hessian $\flatten(\D^2\cL(\theta, z))$ has bounded fourth moments. We then write
\begin{align*}
    \tr(\Cov(\flatten(\D^2 \cL(\theta,z)))) & = \tr(\Cov(\flatten(\Phi''(x_i^T\theta)x_ix_i^T)))\\
    &\leq \tr(\E[\flatten(\Phi''(x_i^T\theta)x_ix_i^T)\flatten(\Phi''(x_i^T\theta)x_ix_i^T)^T])\\
    &\leq \overline{M}_{\Phi, 2}^2 \tr(\E[\flatten(x_ix_i^T)\flatten(x_ix_i^T)^T])\\
    &= \overline{M}_{\Phi, 2}^2 \sum_{j,k=1}^p \E[x_{ij}^2x_{ik}^2]\\
    &\leq \overline{M}_{\Phi, 2}^2 \E\left[ \left(\sum_{j=1}^p x_{ij}\right)^4 \right]\\
    &\leq \overline{M}_{\Phi, 2}^2 \E\left[ \left(x_i^T 1\right)^4 \right],
\end{align*}
where $1$ denotes the all-ones vector. Finally, note that
\begin{align*}
\E\left[ \left(x_i^T 1\right)^4 \right] &\leq \widetilde{C}_4 \E\left[ \left(x_i^T 1\right)^2 \right]^2\\
    &\leq \widetilde{C}_4p^2 \|\Sigma_x\|_2^2,
\end{align*}
implying the desired result.

%%%%%

\subsection{Proof of Proposition~\ref{prop: L, m, M}}
\label{AppPropLM}

For the Lipschitz condition, note that for any $\theta_1, \theta_2\in \real^p$, we have
\begin{align*}
    \| \nabla^2 \cR(\theta_1) - \nabla^2 \cR(\theta_2) \|_2 &=\| \nabla^2 \cL(\theta_1, z) - \nabla^2 \cL(\theta_2, z) \|_2\\
    &= \left\| \E\left[ x_ix_i^T \left( \Phi''(x_i^T\theta_1) - \Phi''(x_i^T\theta_2) \right) \right] \right\|_2\\
    &= \sup_{u\in \mathbb{S}^{p-1}} u^T \E\left[ x_ix_i^T \left( \Phi''(x_i^T\theta_1) - \Phi''(x_i^T\theta_2) \right) \right]u\\
%    &= \sup_{u\in \mathbb{S}^{p-1}}  \E\left[ u^T x_ix_i^T \left( \Phi''(x_i^T\theta_1) - \Phi''(x_i^T\theta_2) \right) u \right]\\
    &= \sup_{u\in \mathbb{S}^{p-1}}  \E\left[ (u^T x_i)^2 \left( \Phi''(x_i^T\theta_1) - \Phi''(x_i^T\theta_2) \right) \right]\\
    &\leq  \sup_{u\in \mathbb{S}^{p-1}}  \E\left[ (u^T x_i)^4\right]^{\frac{1}{2}}\E\left[ \left( \Phi''(x_i^T\theta_1) - \Phi''(x_i^T\theta_2) \right)^2 \right]^\frac{1}{2}\\
    &\leq \sup_{u\in \mathbb{S}^{p-1}}  \sqrt{\widetilde{C}_4}\E\left[ (u^T x_i)^2\right] \sqrt{\overline{M}_{\Phi, 3}} \| \theta_1 - \theta_2\|_2\\
    &\leq \sqrt{\widetilde{C}_4 \overline{M}_{\Phi, 3}}\|\Sigma_x\|_2 \| \theta_1 - \theta_2\|_2,
\end{align*}
where we use the mean value theorem to upper-bound the expectation in the second-to-last inequality.

For any $\theta\in \real^p$, we have
\begin{align*}
    \|\nabla^2 \cR(\theta) \|_2
    &=\| \nabla^2 \cL(\theta, z) \|_2\\
    &= \left\| \E\left[ x_ix_i^T \left( \Phi''(x_i^T\theta)\right) \right] \right\|_2\\
    &= \sup_{u\in \mathbb{S}^{p-1}} u^T \E\left[ x_ix_i^T \left( \Phi''(x_i^T\theta) \right) \right]u\\
%    &= \sup_{u\in \mathbb{S}^{p-1}}  \E\left[ u^T x_ix_i^T \left( \Phi''(x_i^T\theta)  \right) u \right]\\
    &= \sup_{u\in \mathbb{S}^{p-1}}  \E\left[ (u^T x_i)^2 \left( \Phi''(x_i^T\theta) \right) \right]\\
    &\leq  \sup_{u\in \mathbb{S}^{p-1}}  \E\left[ (u^T x_i)^4\right]^{\frac{1}{2}}\E\left[ \left( \Phi''(x_i^T\theta) \right)^2 \right]^\frac{1}{2}\\
    &\leq \sup_{u\in \mathbb{S}^{p-1}}  \sqrt{\widetilde{C}_4}\E\left[ (u^T x_i)^2\right] \overline{M}_{\Phi, 2}\\
    &\leq \overline{M}_{\Phi, 2}\sqrt{\widetilde{C}_4 }\|\Sigma_x\|_2.
\end{align*}

%%%%%

\subsection{Proof of Proposition~\ref{PropConvexity}}
\label{AppPropConvexity}

Suppose $v \in \real^p$ is a unit vector. We write
\begin{align*}
v^T \nabla^2 \cR(\theta) v & = \E\left[(v^T x_i)^2 \cdot \Phi''(x_i^T \theta)\right] \\
& \ge \E\left[(v^T x_i)^2 \cdot b_\tau 1\{|x_i^T \theta| \le \tau\}\right] \\
& = b_\tau \left(\E\left[(v^T x_i)^2\right] - \E\left[(v^T x_i)^2 \cdot 1\{|x_i^T \theta| > \tau\}\right]\right) \\
& \ge b_\tau \left(\lambda_{\min}(\Sigma_x) - \sqrt{\E\left[(v^T x_i)^4\right] \cdot \mprob\left(|x_i^T \theta| > \tau\right)}\right) \\
& \ge b_\tau \left(\lambda_{\min}(\Sigma_x) - \sqrt{\widetilde{C}_4 \|\Sigma_x\|_2^2 \cdot \mprob\left(|x_i^T \theta| > \tau\right)}\right) \\
& \ge \frac{b_\tau}{2} \lambda_{\min}(\Sigma_x),
\end{align*}
where we have used the fact that $\Phi''$ is always nonnegative in the first inequality, applied Cauchy-Schwarz in the second inequality, and used the assumption~\eqref{EqnXCond} in the last inequality.

%%%%%

\section{Proofs about Huber Contamination}
\label{AppHuber}

In this appendix, we provide proofs of the results stated in Section~\ref{sec: app huber}.

%%%%%

\subsection{Bounds on Error Terms}

Our first lemma shows how small the parameters $(\alpha_g, \beta_g, \alpha_h, \beta_h)$ in the robust gradient and Hessian estimates need to be in order to satisfy the assumptions of Theorems~\ref{ThmPure} and~\ref{ThmDamped}.

\begin{lemma*}
\label{lem: thm asmps}
Define
\begin{align*}
    \widehat{\alpha}_g &\defn \min\left\{\frac{m}{64}, \frac{\eta m}{8\sqrt{2M(f(\theta_0)-f(\theta^*))}}, \sqrt{\frac{\eta^2 m}{8\sqrt{2M}}}.\frac{m}{8\sqrt{2M(f(\theta_0)-f(\theta^*))}} \right\},\\
    \widehat{\beta}_g &\defn \min\left\{\frac{\eta}{32}, \frac{\eta m}{8\sqrt{2M(f(\theta_0)-f(\theta^*))}}, \frac{1}{4}\sqrt{\frac{\eta^2 m}{8\sqrt{2M}}} \right\},\\
    \widehat{\alpha}_h &\defn \min\left\{\frac{m^2}{256\eta}, \frac{mM}{4\sqrt{2M(f(\theta_0)-f(\theta^*))}} \right\},\\
    \widehat{\beta}_h &\defn \min\left\{\frac{m}{128}, \frac{M}{2} \right\}.
\end{align*}
Suppose $\alpha_g \leq \widehat{\alpha}_g$, $\beta_g \leq \widehat{\beta}_g$, $\alpha_h \leq \widehat{\alpha}_h$, and $\beta_h \leq \widehat{\beta}_h$. Then the bounds~\eqref{EqnGammaBd} and \eqref{eq: thmpure asmp c2} of Theorem~\ref{ThmPure}, as well as the bounds~\eqref{eq: thmdamped asmp1} and \eqref{eq: thmdamped asmp2} of Theorem~\ref{ThmDamped}, are satisfied.
\end{lemma*}

\begin{proof}
%Let $\alpha_g \leq \widehat{\alpha}_g$, $\beta_g \leq \widehat{\beta}_g$, $\alpha_h \leq \widehat{\alpha}_h$, and $\beta_h \leq \widehat{\beta}_h$.
Under the assumptions, we have
\begin{align*}
    \gamma_g = \frac{2\eta \alpha_g}{m} + \beta_g  & \leq \frac{\eta}{32} + \frac{\eta}{32} = \frac{\eta}{16}, \\
    \gamma_h = \frac{2\eta \alpha_h}{m} + \beta_h & \le \frac{m}{128} +  \frac{m}{128} = \frac{m}{64}, \\
    \frac{2\alpha_g}{m} \sqrt{2M\left(f(\theta_0) - f(\theta^*)\right)} + \beta_g & \le \min\left\{\frac{\eta}{2},  \sqrt{\frac{m}{2}} \cdot \frac{1}{2} \sqrt{\frac{\eta^2}{4\sqrt{2M}}}\right\},\\
    \frac{2\alpha_h}{m} \sqrt{2M\left(f(\theta_0) - f(\theta^*)\right)} + \beta_h & \le M.
\end{align*}
Hence, inequalities~\eqref{EqnGammaBd}, \eqref{eq: thmdamped asmp1}, and \eqref{eq: thmdamped asmp2} are satisfied. 
Using the fact that $\eta \defn \frac{m^2}{8L} \cdot \min\left\{3(1-2\kappa_1), 2 \right\} \leq \frac{m^2}{4L}$, we have $\frac{L}{m^2} \leq \frac{1}{4\eta}$. Then
\begin{align*}
    c_2 &= \eta \left(\frac{4\gamma_g L}{m^2} + \frac{2\gamma_h}{m}\right) + \frac{2L\gamma_g^2}{m^2} + \gamma_g + \frac{2\gamma_g \gamma_h}{m}\\
    &\leq \eta \left(\frac{4\eta}{16}\cdot \frac{1}{4\eta} + \frac{2}{m}\cdot\frac{m}{64}\right) + \frac{2\eta^2}{256}\cdot \frac{1}{4\eta} + \frac{\eta}{16} + \frac{2}{m}\cdot \frac{\eta}{16}\cdot \frac{m}{64}\\
    & < \frac{\eta}{6}\\
    &\leq \frac{m^2}{24L}.
\end{align*}
Hence, inequality~\eqref{eq: thmpure asmp c2} is satisfied.
\end{proof}

In Propositions~\ref{PropErrsHuber} and~\ref{PropErrsHeavy}, below, we derive expressions for $(\alpha_g, \beta_g, \alpha_h, \beta_h)$ for the Huber contamination and heavy-tailed models, which will then allow us to translate the conditions of Lemma~\ref{lem: thm asmps} into assumptions involving the contamination level and/or minimum sample size required for our theoretical results to hold.

We begin with a result concerning the parameters $(\alpha_g, \beta_g, \alpha_h, \beta_h)$ controlling the robust gradient and Hessian errors.

%Define
%\begin{align*}
%    C_6 &= C_1\sqrt{ C_4 \left(\sqrt{L_{\Phi,1, 4}} + L_{\Phi, 2} \right) },\\
%    C_7 &= C_1 \sqrt{C_5 \left( c(\sigma) \sqrt{M_{\Phi, 2, 2}} + \sqrt{c(\sigma)^3 M_{\Phi, 4, 1}}\right)}.
%\end{align*}

\begin{proposition}
\label{PropErrsHuber}
Under the assumptions above, the gradient and Hessian estimates with $\text{{\sc Type}} = \textnormal{Huber}$ returned by Algorithms~\ref{alg: robust gradient estimator} and \ref{alg: robust Hessian estimator}, respectively, satisfy the conditions of Definitions~\ref{defn: robust gradient estimator} and \ref{defn: robust Hessian estimator} with the following parameters:
\begin{align*}
    \alpha_g &= c_1 (\se + \gamma(n,p,\delta,\epsilon)) \sqrt{\|\Sigma_x\|_2 \log p},\\
    \beta_g &=  c_2 (\se + \gamma(n,p,\delta,\epsilon))\sqrt{\|\Sigma_x\|_2 \log p},\\
    \alpha_h &= 0,\\
    \beta_h &= c_3 (\se + \gamma(n,p,\delta,\epsilon)) \|\Sigma_x\|_2 p \sqrt{\log p},
\end{align*}
with probability at least $1-\delta$.
\end{proposition}

\begin{proof}
By Lemma~\ref{lem: grad est GLM}, the true distribution of the gradients $\nabla \cL(\theta, z)$ has bounded fourth moments. Moreover,
\begin{align*}
    \|\Cov(\D \cL(\theta,z))\|_2
    &\leq C_1 \|\Sigma_x\|_2 \left(\sqrt{L_{\Phi, 4}} + L_{\Phi, 2} \right) \|\theta-\theta^*\|_2^2 \nonumber\\
    &\ \ \ + C_2 \|\Sigma_x\|_2 \left(B_{\Phi, 2} + \sqrt{B_{\Phi, 4}} + c(\sigma) \sqrt{M_{\Phi, 2, 2}} + \sqrt{c(\sigma)^3 M_{\Phi, 4, 1}}\right).
\end{align*}
Plugging the above bound into inequality~\eqref{eq: lem grad est huber} of Lemma~\ref{lem: grad est huber}, we obtain
\begin{align*}
    \| g(\theta) - \E[\D \cL(\theta,z)] \|_2 & \leq C_1' (\se + \gamma(n,p,\delta,\epsilon)) \sqrt{ \|\Cov(\D \cL(\theta,z)) \|_2 \log p}\\
    & \leq c_1 (\se + \gamma(n,p,\delta,\epsilon)) \sqrt{\|\Sigma_x\|_2 \log p} \cdot \|\theta-\theta^*\|_2 \\
    & \qquad + c_2 (\se + \gamma(n,p,\delta,\epsilon))\sqrt{\|\Sigma_x\|_2 \log p}.
\end{align*}
Hence, the gradient estimate returned by Algorithm~\ref{alg: robust gradient estimator} satisfies Definition~\ref{defn: robust gradient estimator} with $\alpha_g= c_1 (\se + \gamma(n,p,\delta,\epsilon)) \sqrt{\|\Sigma_x\|_2 \log p}$ and $\beta_g =  c_2 (\se + \gamma(n,p,\delta,\epsilon))\sqrt{\|\Sigma_x\|_2 \log p}$.

By Lemma~\ref{lem: Hess est GLM}, the true distribution of the flattened Hessian $\flatten(\D^2 \cL(\theta,z))$ has bounded fourth moments. Moreover, combining Lemma~\ref{lem: Hess est GLM} with Lemma~\ref{lem: Hessian est huber}, we obtain
\begin{align*}
     \| H(\theta) - \E[\D^2 \cL(\theta,z)]  \|_2 
     &\leq C_2' (\se + \gamma(n,p,\delta,\epsilon))\sqrt{\|\Cov(\flatten(\D^2\cL(\theta, z)))) \|_2 \log p}\\
     &\leq c_3 (\se + \gamma(n,p,\delta,\epsilon)) \|\Sigma_x\|_2 p \sqrt{\log p}.
\end{align*}
Hence, the Hessian estimate returned by Algorithm~\ref{alg: robust Hessian estimator} satisfies Definition~\ref{defn: robust Hessian estimator} with $\alpha_h = 0$ and $\beta_h = c_3 (\se + \gamma(n,p,\delta,\epsilon)) \|\Sigma_x\|_2 p \sqrt{\log p}$.
\end{proof}

%%%%%

\subsection{Proof of Theorem~\ref{ThmGLMHuber}}
\label{AppThmGLMHuber}

We will apply Theorems~\ref{ThmPure} and~\ref{ThmDamped} to show that Algorithm~\ref{alg: robust Hessian estimator} returns $\thetahat_T$ such that $\|\theta_T - \theta^*\|_2 \leq \frac{12c_2}{m} = O\left(\epsilon\log(p) + \sqrt{\epsilon\log(p)}\right)$.

Under the assumption that $\se + \gamma\left(n,p,\delta,\epsilon\right)< \widehat{\gamma}$, and using Proposition~\ref{PropErrsHuber}, the assumptions of Lemma~\ref{lem: thm asmps} are satisfied. Note that the assumptions of Lemma~\ref{lem: rob est Hub} are likewise satisfied by the condition~\eqref{EqnCondZeta}.
Applying Theorem~\ref{ThmDamped}, the risk $\cR(\theta_t)$ is reduced by at least $\frac{\gamma}{2}$ in each step of the damped Newton phase of the algorithm. Hence, the number of such iterations cannot exceed
\begin{align}\label{eq: T_damp}
    T_{damp} \defn  \frac{\cR(\theta_0) - \cR(\theta^*)}{\gamma/2}.
\end{align}
Define 
\begin{align}\label{eq: T_pure}
    T_{pure} \defn \log_2\log_2\left(\frac{6c_2L}{m^2}\right).
\end{align}
Applying Theorem~\ref{ThmPure}, we observe that after $T_{pure}$ iterations in the pure Newton phase, we have $\frac{m}{L} \left(\frac{1}{2}\right)^{2^t} < \frac{6c_2}{m}$. Therefore, from inequality~\eqref{EqnContract}, we have $\|\thetahat - \theta^*\|_2 \leq \frac{12c_2}{m}$. Combining inequalities~\eqref{eq: T_damp} and~\eqref{eq: T_pure}, we obtain the bound on the total number of iterations $T$.

From the preceding analysis on the robust gradient and Hessian estimators, observe that $\eta, L$, and $m$ are independent of $\epsilon$ and $p$, while $\gamma_g$ is $O(\sqrt{\epsilon\log p})$ and $\gamma_h$ is $O(p \sqrt{\epsilon \log p})$.
%Hence, $c_2$ is $ O\left(\epsilon\log(p) + \sqrt{\epsilon\log(p)}\right)$. 
Hence, from inequality~\eqref{EqnContract}, we have $\|\theta_T - \theta^*\|_2 = O\left(p \sqrt{\epsilon \log p}\right)$.

We now compute the error probability of the algorithm via a union bound. For each of the $T$ gradient and Hessian calculations, we have a possible error of $\delta$. Furthermore, each call of backtracking linesearch incurs a possible error from the robust estimates, by Lemma~\ref{lem: rob est Hub}; once at $\theta_t$ and once for each value of $\alpha$ used in the linesearch. This is a total of $2T_{pure}$ evaluations for the pure Newton steps,  and a maximum of $T_{damp}\left(1 + \Big\lceil\frac{\log\left(\frac{m}{M}\left(\frac{1}{2} - \kappa_1\right)\right)}{\log \kappa_2}\Big\rceil\right)$ evaluations for the damped Newton steps. Thus, the overall probability of error is at most
\begin{equation*}
2T\delta + 2T_{pure}\delta + T_{damp}\left(1 + \Bigg\lceil\frac{\log\left(\frac{m}{M}\left(\frac{1}{2} - \kappa_1\right)\right)}{\log \kappa_2}\Bigg\rceil\right)\delta
\leq T\delta \left(5 + \Bigg\lceil\frac{\log\left(\frac{m}{M}\left(\frac{1}{2} - \kappa_1\right)\right)}{\log \kappa_2}\Bigg\rceil \right).
\end{equation*}

%%%%%

\section{Proofs about Heavy-Tailed Contamination}
\label{AppHeavy}

In this appendix, we provide proofs of the results stated in Section~\ref{sec: app heavy-tail}.

\subsection{Bounds on Error Terms}

The first result concerns the parameters $(\alpha_g, \beta_g, \alpha_h, \beta_h)$, which control the robust gradient and Hessian errors.
%Define
%\begin{align*}
%    C_{10} &= 11\sqrt{ C_4 \left(\sqrt{L_{\Phi,1, 4}} + L_{\Phi, 2} \right) },\\
%    C_{11} &= 11 \sqrt{C_5 \left( c(\sigma) \sqrt{M_{\Phi, 2, 2}} + \sqrt{c(\sigma)^3 M_{\Phi, 4, 1}}%\right)}.
%\end{align*}

\begin{proposition}
\label{PropErrsHeavy}
Under the assumptions above, the gradient and Hessian estimates with $\text{{\sc Type}} = \textnormal{Heavy-tail}$ returned by Algorithms~\ref{alg: robust gradient estimator} and \ref{alg: robust Hessian estimator}, respectively, satisfy the conditions of Definitions~\ref{defn: robust gradient estimator} and~\ref{defn: robust Hessian estimator} with the following parameters:
\begin{align*}
    \alpha_g &= c_1 \sqrt{ \frac{p\log(1.4/\delta)}{n} },\\
    \beta_g &=  c_2 \sqrt{ \frac{p\log(1.4/\delta)}{n} },\\
    \alpha_h &= 0,\\
    \beta_h &= c_3 \|\Sigma_x\|_2 p \sqrt{\frac{\log(1.4/\delta)}{n}},
\end{align*}
with probability at least $1-\delta$.
\end{proposition}

\begin{proof}
By Lemma~\ref{lem: grad est GLM}, the distribution of the gradients $\nabla \cL(\theta, z)$ has bounded fourth moments. Moreover,
\begin{align*}
    \|\Cov(\D \cL(\theta,z))\|_2
    &\leq C_1 \|\Sigma_x\|_2 \left(\sqrt{L_{\Phi, 4}} + L_{\Phi, 2} \right) \|\theta-\theta^*\|_2^2 \nonumber\\
    &\ \ \ + C_2 \|\Sigma_x\|_2 \left(B_{\Phi, 2} + \sqrt{B_{\Phi, 4}} + c(\sigma) \sqrt{M_{\Phi, 2, 2}} + \sqrt{c(\sigma)^3 M_{\Phi, 4, 1}}\right).
\end{align*}
Plugging this bound into inequality~\eqref{eq: lem grad est heavy tail} of Lemma~\ref{lem: grad est heavy-tail}, we obtain
\begin{align*}
    \| g(\theta) - \E[\D \cL(\theta,z)] \|_2 
    &\leq 11\sqrt{ \frac{\tr(\Cov(\D \cL(\theta,z)))\log(1.4/\delta)}{n}}\\
    &\leq 11\sqrt{ \frac{p\Cov(\D \cL(\theta,z))\log(1.4/\delta)}{n}}\\
    &\leq c_1 \sqrt{ \frac{p\log(1.4/\delta)}{n} }\|\theta-\theta^*\|_2 + c_2  \sqrt{ \frac{p\log(1.4/\delta)}{n} }.
\end{align*}
Hence, the gradient estimate returned by Algorithm~\ref{alg: robust gradient estimator} satisfies Definition~\ref{defn: robust gradient estimator} with $\alpha_g= c_1 \sqrt{ \frac{p\log(1.4/\delta)}{n} }$ and $\beta_g = c_2 \sqrt{ \frac{p\log(1.4/\delta)}{n} }$.

By Lemma~\ref{lem: Hess est GLM}, the distribution of the flattened Hessian $\flatten(\D^2 \cL(\theta,z))$ has bounded fourth moments. Moreover, combining Lemma~\ref{lem: Hess est GLM} with Lemma~\ref{lem: Hessian est heavy-tail}, we obtain
\begin{align*}
     \| H(\theta) - \E[\D^2 \cL(\theta,z)]  \|_2 
     &\leq C_3\sqrt{ \frac{\tr(\Cov(\flatten(\D^2 \cL(\theta,z))))\log(1.4/\delta)}{n}}\\
     &\leq c_3 \|\Sigma_x\|_2 p \sqrt{\frac{\log(1.4/\delta)}{n}}.
\end{align*}
Hence, the Hessian estimate returned by Algorithm~\ref{alg: robust Hessian estimator} satisfies Definition~\ref{defn: robust Hessian estimator} with $\alpha_h = 0$ and $\beta_h = c_3 \|\Sigma_x\|_2 \sqrt{\frac{\log(1.4/\delta)}{n}}$.
\end{proof}

%%%%%

\subsection{Proof of Theorem~\ref{ThmGLMHeavy}}
\label{AppThmGLMHeavy}

We will follow a similar outline as in the proof of Theorem~\ref{ThmGLMHuber}.

Using the assumption on $n$ and Proposition~\ref{PropErrsHeavy}, it is straightforward to verify that the conditions of Lemma~\ref{lem: thm asmps} are satisfied. Furthermore, the conditions of Lemma~\ref{lem: rob est heavy} are satisfied by inequality~\eqref{eq: ThmGLMHeavy n asmp}, as well.
Applying Theorem~\ref{ThmDamped}, the risk $\cR(\theta_t)$ is reduced by at least $\frac{\gamma}{2}$ in each step of the damped Newton phase of the algorithm. Hence, the number of such iterations cannot exceed $T_{damp}$, defined as in equation~\eqref{eq: T_damp}.
Applying Theorem~\ref{ThmPure}, we observe that after $T_{pure}$ iterations (defined as in equation~\eqref{eq: T_pure}) in the pure Newton phase, we have $\frac{m}{L} \left(\frac{1}{2}\right)^{2^t} < \frac{6c_2}{m}$. Therefore, from inequality~\eqref{EqnContract}, we have $\|\thetahat - \theta^*\|_2 \leq \frac{12c_2}{m}$. Combining inequalities~\eqref{eq: T_damp} and~\eqref{eq: T_pure}, we obtain the bound on the total number of iterations $T$.

From the preceding analysis on the robust gradient and Hessian estimators (cf.\ Proposition~\ref{PropErrsHeavy}), observe that $\gamma_g = O\left(\sqrt{\frac{p}{n}}\right)$ and $\gamma_h = O\left(\sqrt{\frac{p^2}{n}}\right)$. Hence, $c_2$ is $ O\left(\sqrt{\frac{p^2}{n}}\right)$. 
From inequality~\eqref{EqnContract}, we then have $\|\theta_T - \theta^*\|_2 = O\left(\sqrt{\frac{p^2}{n}}\right)$.

Computing the error probability of the algorithm via a union bound is the same as in Theorem~\ref{ThmGLMHuber} with the use of appropriate gradient, Hessian, and robust estimates. 

%%%%%

\section*{Acknowledgments}

The work of EI was supported by the Cantab Capital Institute for the Mathematics of Information via the Philippa Fawcett Internship programme (Faculty of Mathematics, University of Cambridge). The authors thank the associate editor and anonymous reviewers for their helpful feedback, which improved the quality of the paper.

\bibliographystyle{plain}
\bibliography{refs}

\end{document}